\documentclass{article}

\usepackage{microtype}
\usepackage{graphicx}
\usepackage{subfigure}
\usepackage{booktabs} 
\usepackage{amsmath}
\usepackage{amssymb}
\usepackage{bbm}
\usepackage{bm}
\usepackage{amsthm}
\PassOptionsToPackage{hyphens}{url}
\usepackage{hyperref}

\newcommand{\EE}[1]{\mathbb{E} \left[ #1 \right]}
\newcommand{\EEsub}[2]{\mathop{\mathbb{E}}_{#2} \left[ #1 \right]}

\newcommand{\argmin}{\arg\,\min}
\DeclareMathOperator*{\vecc}{vec}

\newtheorem{thm}{Theorem}
\newtheorem*{thm*}{Theorem}
\newtheorem{lem}[thm]{Lemma}
\newtheorem*{lem*}{Lemma}

\newtheorem{deff}{Definition}

\usepackage[accepted]{icml2018}
\icmltitlerunning{Learning to Reweight Examples for Robust Deep Learning}
\begin{document}

\twocolumn[
\icmltitle{Learning to Reweight Examples for Robust Deep Learning}
\icmlsetsymbol{equal}{*}

\begin{icmlauthorlist}
\icmlauthor{Mengye Ren}{atg,to}
\icmlauthor{Wenyuan Zeng}{atg,to}
\icmlauthor{Bin Yang}{atg,to}
\icmlauthor{Raquel Urtasun}{atg,to}
\end{icmlauthorlist}

\icmlaffiliation{atg}{Uber Advanced Technologies Group, Toronto ON, CANADA}
\icmlaffiliation{to}{Department of Computer Science, University of Toronto, Toronto ON, CANADA}

\icmlcorrespondingauthor{Mengye Ren}{mren3@uber.com}

\icmlkeywords{Example reweighting, meta-Learning, deep learning, machine learning}
\vskip 0.3in
]
\def\arxiv{1}

\printAffiliationsAndNotice{}  

\begin{abstract}

Deep neural networks have been shown to be very powerful modeling tools for many supervised learning
tasks involving complex input patterns.  However, they can also easily overfit to training set
biases and label noises.  In addition to various regularizers, example reweighting algorithms are
popular solutions to these problems, but they  require careful tuning of additional hyperparameters,
such as example mining schedules and regularization hyperparameters. In contrast to past reweighting
methods, which typically consist of functions of the cost value of each example, in this work we
propose a novel meta-learning algorithm that learns to assign weights to training examples based on
their gradient directions. To determine the example weights, our method performs a meta gradient
descent step on the current mini-batch example weights (which are initialized from zero) to minimize
the loss on a clean unbiased validation set. Our proposed method can be easily implemented on any
type of deep network, does not require any additional hyperparameter tuning, and achieves impressive
performance on class imbalance and corrupted label problems where only a small amount of clean
validation data is available.

\end{abstract}


\section{Introduction}

Deep neural networks (DNNs) have been widely used for machine learning applications due to their
powerful capacity for modeling complex input patterns. Despite their  success, it has been shown
that DNNs are prone to training set biases, i.e. the training set  is drawn from a joint
distribution $p(x, y)$ that is different from the distribution $p(x^v, y^v)$ of the evaluation set.
This distribution mismatch could have many different forms.  Class imbalance in the training set is
a very common example. In applications such as object detection in the context of autonomous
driving, the vast majority of the training data is composed of standard  vehicles but models also
need to recognize rarely seen classes such as emergency vehicles or animals with very high accuracy.
This will sometime lead to biased training models that do not perform well in practice.

Another popular type of training set bias is label noise. To train a reasonable supervised deep
model, we ideally need a large dataset with high-quality labels, which require many passes of
expensive human quality assurance (QA). Although coarse labels are cheap and of high availability,
the presence of noise will hurt the model performance, e.g. \citet{rethink} has shown that a standard
CNN can fit any ratio of label flipping noise in the training set and eventually leads to poor
generalization performance.

Training set biases and misspecification can sometimes be addressed with dataset resampling
\cite{smote}, i.e. choosing the correct proportion of labels to train a network on, or more
generally by assigning a weight to each example and minimizing a weighted training loss. The example
weights are typically calculated based on the training loss, as in many classical algorithms such as
AdaBoost \cite{adaboost}, hard negative mining \cite{hardneg}, self-paced learning
\cite{kumar10selfpaced}, and other more recent work \cite{chang17activebias,jiang17mentornet}.

However, there exist two contradicting ideas in training loss based approaches. In noisy label
problems, we prefer examples with smaller training losses as they are more likely to be clean
images; yet in class imbalance problems, algorithms such as hard negative mining \cite{hardneg}
prioritize examples with higher training loss since they are more likely to be the minority class.
In cases when the training set is both imbalanced and noisy, these existing methods would have the
wrong model assumptions. In fact, without a proper definition of an unbiased test set, solving the
training set bias problem is inherently ill-defined. As the model cannot distinguish the right from
the wrong, stronger regularization can usually work surprisingly well in certain synthetic noise
settings. Here we argue that in order to learn general forms of training set biases, it is necessary
to have a small unbiased validation to guide training. It is actually not uncommon to construct a
dataset with two parts - one relatively small but very accurately  labeled, and another massive but
coarsely labeled. Coarse labels can come from inexpensive crowdsourcing services   or weakly
supervised data \cite{cityscapes,ILSVRC15,webly}.

Different from existing training loss based approaches, we follow a meta-learning paradigm and model
the most basic assumption instead: \textit{the best example weighting should minimize the loss of a
set of unbiased clean validation examples that are consistent with the evaluation procedure}.
Traditionally, validation is performed at the end of training, which can be prohibitively expensive
if we treat the example weights as some hyperparameters to optimize; to circumvent this, we perform
validation at \textit{every} training iteration to dynamically determine the example weights of the
current batch. Towards this goal, we propose  an online reweighting method that leverages an
additional small validation set and adaptively assigns importance weights to examples in every
iteration. We experiment with both class imbalance and corrupted label problems and find that our
approach significantly increases the robustness to training set biases.


\section{Related Work}
The idea of weighting each training example has been well studied in the literature. Importance
sampling \cite{importantsample}, a classical method in statistics, assigns weights to samples in
order to match one distribution to another. Boosting algorithms such as AdaBoost \cite{adaboost},
select harder examples to train subsequent classifiers. Similarly, hard example mining
\cite{hardneg}, downsamples the majority class and exploits the most difficult examples. Focal loss
\cite{focal} adds a soft weighting scheme that emphasizes harder examples.

Hard examples are not always preferred in the presence of outliers and noise processes. Robust loss
estimators typically downweigh examples with high loss. In self-paced learning
\cite{kumar10selfpaced}, example weights are obtained through optimizing the weighted training loss
encouraging learning easier examples first. In each step, the learning algorithm jointly solves a
mixed integer program that iterates optimizing over model parameters and binary example weights.
Various regularization terms  on the example weights have since been proposed to prevent overfitting
and trivial solutions of assigning weights to be all zeros \cite{kumar10selfpaced,spaco,spcl}.
\citet{wang17reweight} proposed a Bayesian method that infers the example weights as latent
variables. More recently, \citet{jiang17mentornet} proposed to use a meta-learning LSTM to output
the weights of the examples based on the training loss. Reweighting examples is also related to
curriculum learning \cite{bengio09curriculum}, where the model reweights among many available tasks.
Similar to self-paced learning, typically it is beneficial to start with easier examples.

One crucial advantage of reweighting examples is robustness against training set bias. There has
also been a multitude of prior studies on class imbalance problems, including using dataset
resampling \cite{smote,dong17imbalance}, cost-sensitive weighting
\cite{costsensitive,costsensitivedeep}, and structured margin based objectives \cite{lmle}.
Meanwhile, the noisy label problem has been thoroughly studied by the learning theory community
\cite{natarajan13noisy,noisytheory} and practical methods have also been proposed
\cite{reed14noisy,sukhbaatar14convnoise,xiao15noisy,azadi16air,goldberger17noise,
li17noisydistill,jiang17mentornet,vahdat17crf,glc}.  In addition to corrupted data,
\citet{kohL17influence,datapoison} demonstrate the possibility of a dataset adversarial attack (i.e.
dataset poisoning).

Our method improves the training objective through a weighted loss rather than an average loss and
is an instantiation of meta-learning \cite{metalearn,lakemetalearn,l2l}, i.e. learning to learn
better. Using validation loss as the meta-objective has been explored in recent meta-learning
literature for few-shot learning \cite{ravi2017oneshot,metafewshot,hpernet}, where only a handful of
examples are available for each class. Our algorithm also resembles MAML \cite{maml} by taking one
gradient descent step on the meta-objective for each iteration. However, different from these
meta-learning approaches, our reweighting method does not have any additional hyper-parameters and
circumvents an expensive offline training stage. Hence, our method can work in an online fashion
during regular training.

\section{Learning to Reweight Examples}

In this section, we derive our model from a meta-learning objective towards an online approximation
that can fit into any regular supervised training. We give a practical implementation suitable for
any deep network type and  provide theoretical guarantees under mild conditions that our algorithm
has a convergence rate of $O(1/\epsilon^2)$. Note that  this is the same as that of stochastic
gradient descent (SGD).

\subsection{From a meta-learning objective to an online approximation}

Let $(x,y)$ be an input-target pair, and $\{(x_i, y_i), 1 \le i \le N\}$ be the training set. We
assume that there is a small unbiased and clean validation set $\{(x^v_i, y^v_i), 1 \le i \le M\}$,
and $M \ll N$. Hereafter, we will use superscript $v$ to denote validation set and subscript $i$ to
denote the $i^{th}$ data. We also assume that the training set contains the validation set;
otherwise, we can always add this small validation set into the training set and leverage more
information during training.

Let $\Phi (x,\theta)$ be our neural network model, and $\theta$ be the model parameters. We consider
a loss function $C(\hat{y}, y)$ to minimize during training, where $\hat{y} = \Phi (x, \theta)$.

In standard training, we aim to minimize the expected loss for the training set: $\frac{1}{N}
\sum_{i=1}^N C(\hat{y}_i, y_i)=\frac{1}{N} \sum_{i=1}^N f_i(\theta)$, where each input example is
weighted equally, and $f_i(\theta)$ stands for the loss function associating with data $x_i$. Here
we aim to learn a reweighting of the inputs, where we minimize a weighted loss:
\begin{equation}
\label{eq:theta_star}
\theta^*(w) = \argmin_\theta \sum_{i=1}^N w_i f_i(\theta),
\end{equation}
with $w_i$  unknown upon beginning. Note that $\{w_i\}_{i=1}^N$ can be understood as  training
hyperparameters, and the optimal selection of $w$ is based on its validation performance:
\begin{equation}\label{eq:w_star}
w^* = \argmin_{w, w \ge 0} \frac{1}{M} \sum_{i=1}^M f_i^v(\theta^*(w)).
\end{equation}
It is necessary that $w_i \ge 0$ for all $i$, since  minimizing the negative training loss can
usually result in unstable behavior.

\paragraph{Online approximation} Calculating the optimal $w_i$ requires two nested loops of
optimization, and every single loop can be very expensive. The motivation of our approach is to adapt
online $w$ through a single optimization loop. For each training iteration, we inspect the descent
direction of some training examples locally on the training loss surface and reweight them
according to their similarity to the descent direction of the validation loss surface.

For most training of deep neural networks, SGD or its variants are used to  optimize such loss
functions. At every step $t$ of training, a mini-batch of training examples $\{(x_i, y_i), 1 \le i
\le n\}$ is sampled, where $n$ is the mini-batch size, $n \ll N$. Then the parameters are adjusted
according  to the descent direction of the expected loss on the mini-batch. Let's consider vanilla
SGD:
\begin{align}
\theta_{t+1} &= \theta_t - \alpha \nabla \left( \frac{1}{n}\sum_{i=1}^n f_i(\theta_t) \right),
\end{align}
where $\alpha$ is the step size.

We want to understand what would be the impact of training example $i$ towards  the performance of
the validation set at training step $t$. Following a similar analysis to \citet{kohL17influence}, we
consider perturbing the weighting by $\epsilon_i$ for each training example in the mini- batch,
\begin{align}
f_{i,\epsilon}(\theta) &= \epsilon_i f_i(\theta),\\
\hat{\theta}_{t+1}(\epsilon) &= \theta_t - \alpha \nabla 
    \sum_{i=1}^n f_{i,\epsilon}(\theta)\Bigr|_{\theta=\theta_t}.
\end{align}
We can then look for the optimal $\epsilon^*$ that minimizes the validation loss $f^v$ locally at step $t$:
\begin{align}
\epsilon^*_t = \argmin_\epsilon \frac{1}{M} \sum_{i=1}^M f^v_i(\theta_{t+1}(\epsilon)).
\end{align}
Unfortunately,  this can still be quite time-consuming. To get a cheap estimate of $w_i$ at step
$t$, we take a single gradient descent step on a mini-batch of validation samples wrt.
$\epsilon_t$, and then rectify the output to get a non-negative weighting:
\begin{align}
\label{eq:meta-gradient}
u_{i,t} &= -\eta \frac{\partial}{\partial \epsilon_{i,t}} \frac{1}{m} 
            \sum_{j=1}^m f_j^v(\theta_{t+1}(\epsilon)) \Bigr|_{\epsilon_{i,t}=0},\\
\tilde{w}_{i,t} &= \max(u_{i,t}, 0).
\end{align}
where $\eta$ is the descent step size on $\epsilon$.

To match the original training step size, in practice, we can consider normalizing the weights of
all examples in a training batch so that they sum up to one. In other words, we choose to have
a hard constraint within the set $\{w: \lVert w \rVert_1 = 1 \} \cup \{0\}$.
\begin{align}
w_{i,t} = \frac{\tilde{w}_{i, t}}{(\sum_j \tilde{w}_{j, t}) + \delta(\sum_j \tilde{w}_{j, t})},
\end{align}

where $\delta(\cdot)$ is to prevent the degenerate case when all $w_i$'s in a mini-batch are zeros,
i.e. $\delta(a) = 1$ if $a = 0$, and equals to $0$ otherwise. Without the batch-normalization step,
it is possible that the algorithm modifies its effective learning rate of the training progress, and
our one-step look ahead may be too conservative in terms of the choice of learning rate
\cite{shorthorizon}. Moreover, with batch normalization, we effectively cancel the meta learning
rate parameter $\eta$.

\subsection{Example: learning to reweight examples in a multi-layer perceptron network}
In this section, we study how to compute $w_{i,t}$ in a multi-layer perceptron (MLP) network. One of
the core steps is to compute the gradients of the validation loss wrt. the local perturbation
$\epsilon$, We can consider a multi-layered network where we have parameters for each layer $\theta
= \{\theta_l\}_{l=1}^L$, and at every layer, we first compute $z_l$ the pre-activation, a weighted
sum of inputs to the layer, and afterwards we apply a non-linear activation function $\sigma$ to
obtain $\tilde{z}_l$ the post-activation:
\begin{align}
z_l &= \theta_l^\top \tilde{z}_{l-1},\\
\tilde{z}_l &= \sigma(z_l).
\end{align}
During backpropagation, let $g_l$ be the gradients of loss wrt. $z_l$, and the gradients wrt.
$\theta_l$ is given by $\tilde{z}_{l-1} g_l^\top$.
We can further express the gradients towards $\epsilon$ as a sum of local dot products.
\begin{align}
\begin{split}
&\frac{\partial}{\partial \epsilon_{i,t}} \mathbb{E}\left[ f^v(\theta_{t+1}(\epsilon))
\Bigr|_{\epsilon_{i,t}=0} \right]\\
\propto&-\frac{1}{m}\sum_{j=1}^m 
\frac{\partial f_j^v(\theta)}{\partial \theta}\Bigr|_{\theta=\theta_t}^\top
\frac{\partial f_i(\theta)}{\partial \theta}\Bigr|_{\theta=\theta_t}\\
=&-\frac{1}{m}\sum_{j=1}^m \sum_{l=1}^L
(\tilde{z}^v_{j,l-1}{}^\top
\tilde{z}_{i,l-1})
(g^v_{j,l}{}^\top g_{i,l}).
\label{eq:sim}
\end{split}
\end{align}
\if\arxiv1
Detailed derivations can be found in Appendix~\ref{sec:mlp_derive}. 
\else
Detailed derivations can be found in Supplementary Materials.
\fi
Eq.~\ref{eq:sim} suggests that
the meta-gradient on $\epsilon$ is composed of the sum of the products of two terms: $z^\top z^v$
and $g^\top g^v$. The first dot product computes the similarity between the training and validation
inputs to the layer, while the second computes the similarity between the training and validation
gradient directions. In other words, suppose that a pair of training and validation examples are
very similar, and they also provide similar gradient directions, then this training example is
helpful and should be up-weighted, and conversely, if they provide opposite gradient directions,
this training example is harmful and should be downweighed.

\subsection{Implementation using automatic differentiation}
\begin{figure}[t]
\centering
\includegraphics[width=\columnwidth,trim={0cm 8.9cm 15cm 0},clip]{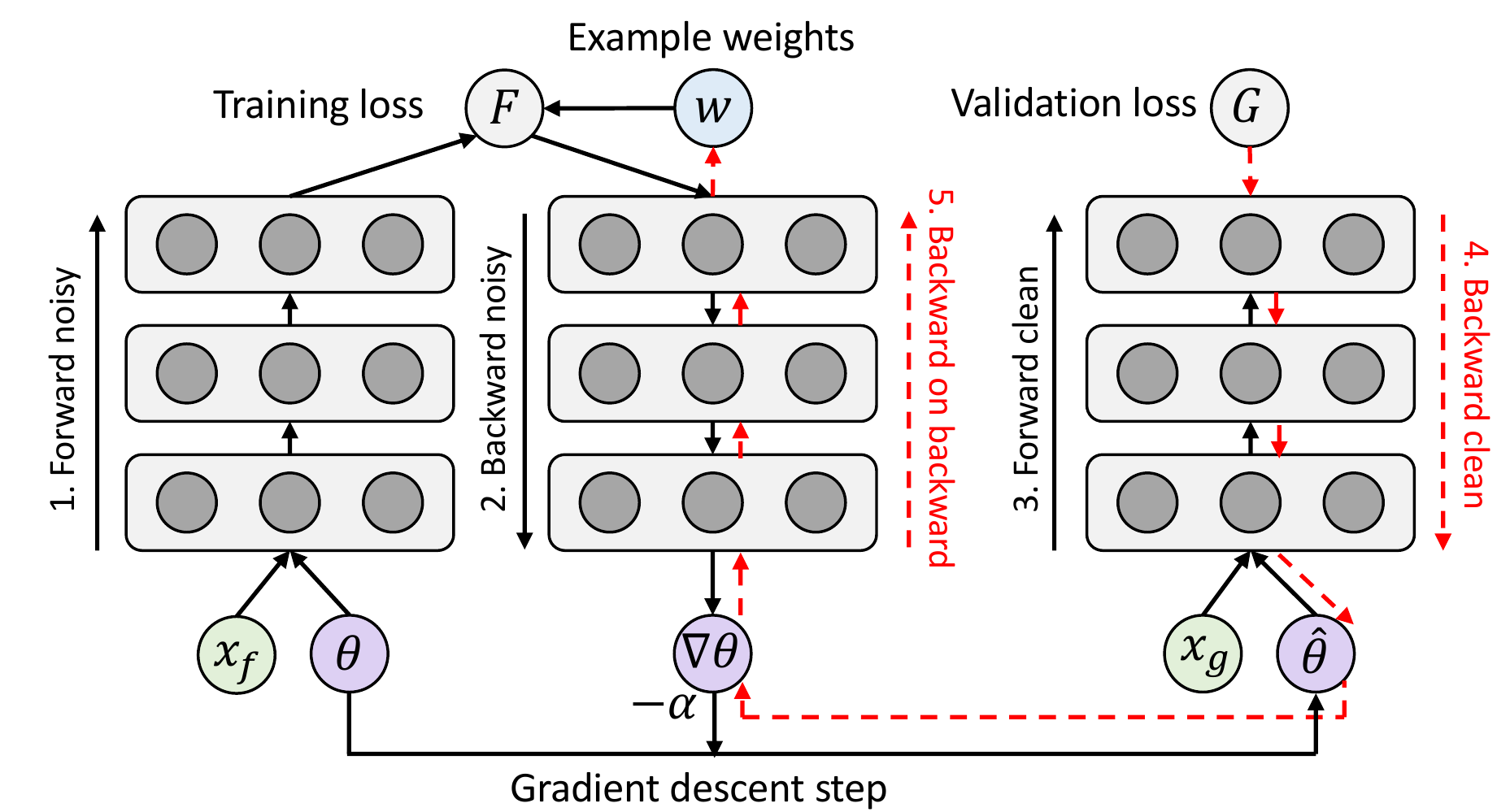}
\vspace{-0.1in}
\caption{Computation graph of our algorithm in a deep neural network, which can be efficiently implemented using second order automatic differentiation.}
\label{fig:comp_graph}
\end{figure} In an MLP and a CNN, the unnormalized weights can be calculated based on
the sum of the correlations of layerwise activation gradients and input activations. In more general
networks, we can leverage automatic differentiation techniques to compute the gradient of the
validation loss wrt. the example weights of the current batch. As shown in
Figure~\ref{fig:comp_graph}, to get the gradients of the example weights, one needs to first unroll
the gradient graph of the training batch, and then use backward-on-backward automatic
differentiation to take a second order gradient pass (see Step 5 in Figure~\ref{fig:comp_graph}). We
list detailed step-by-step pseudo-code in Algorithm~\ref{alg:ad}. This implementation can be
generalized to any deep learning architectures and can be very easily implemented using popular deep
learning frameworks such as TensorFlow \cite{tensorflow}.

\begin{minipage}{\columnwidth}
\begin{algorithm}[H]
\caption{Learning to Reweight Examples using Automatic Differentiation}
\label{alg:ad}
\begin{algorithmic}[1]
\REQUIRE $\theta_0$, $\mathcal{D}_f$, $\mathcal{D}_g$, $n$, $m$
\ENSURE $\theta_T$
\FOR{$t=0$ ... $T-1$}
\STATE $\{X_f, y_f\} \gets$ \text{SampleMiniBatch}($\mathcal{D}_f$, $n$)
\STATE $\{X_g, y_g\} \gets$ \text{SampleMiniBatch}($\mathcal{D}_g$, $m$)
\STATE $\hat{y}_f \gets \text{Forward}(X_{f}, y_{f}, \theta_t)$
\STATE $\epsilon \gets 0$; $l_f \gets \sum_{i=1}^n \epsilon_i C(y_{f,i}, \hat{y}_{f,i})$
\STATE $\nabla \theta_t \gets \text{BackwardAD}(l_f, \theta_t)$
\STATE $\hat{\theta}_t \gets \theta_t - \alpha \nabla \theta_t$
\STATE $\hat{y}_g \gets \text{Forward}(X_{g}, y_{g}, \hat{\theta}_t)$
\STATE $l_g \gets \frac{1}{m} \sum_{i=1}^m C(y_{g,i}, \hat{y}_{g,i})$
\STATE $\nabla \epsilon \gets \text{BackwardAD}(l_g, \epsilon)$  \label{lst:line:bb}
\STATE $\tilde{w} \gets \max(-\nabla \epsilon, 0)$; $w \gets \frac{\tilde{w}}{\sum_j \tilde{w} + \delta(\sum_j \tilde{w})}$
\STATE $\hat{l}_f \gets \sum_{i=1}^n w_i C(y_i, \hat{y}_{f,i})$
\STATE $\nabla \theta_t \gets \text{BackwardAD}(\hat{l}_f, \theta_t)$
\STATE $\theta_{t+1} \gets \text{OptimizerStep}(\theta_t, \nabla \theta_t)$
\ENDFOR
\end{algorithmic}
\end{algorithm}
\end{minipage}

\paragraph{Training time} Our automatic reweighting method will introduce a constant factor of
overhead. First, it requires two full forward and backward passes of the network on training and
validation respectively, and then another backward on backward pass (Step 5 in
Figure~\ref{fig:comp_graph}), to get the gradients to the example weights, and finally a backward
pass to minimize the reweighted objective. In modern networks, a backward-on-backward pass usually
takes about the same time as a forward pass, and therefore compared to regular training, our method
needs approximately 3$\times$ training time; it is also possible to reduce the batch size of the
validation pass for speedup. We expect that it is worthwhile to spend the extra time to avoid the
irritation of choosing early stopping, finetuning schedules, and other hyperparameters.

\subsection{Analysis: convergence of the reweighted training}
Convergence results of SGD based optimization methods are well-known \cite{svrg}. However it is
still meaningful to establish a convergence result about our method since it involves optimization
of two-level objectives (Eq. \ref{eq:theta_star}, \ref{eq:w_star}) rather than one, and we further
make some first-order approximation by introducing Eq. \ref{eq:meta-gradient}. Here, we show
theoretically that our method converges to the critical point of the validation loss function under
some mild conditions, and we also give its convergence rate. More detailed proofs can be found in
the 
\if\arxiv1
Appendix~\ref{sec:lemproof},~\ref{sec:thmproof}.
\else
Supplementary Materials.
\fi

\begin{deff}\label{deff:lipandbound}
A function $f(x): \mathbb{R}^d \to \mathbb{R}$ is said to be Lipschitz-smooth with constant $L$ if
\begin{align*}
\lVert \nabla f(x) - \nabla f(y) \rVert \leq L \lVert x - y \rVert, \forall x, y \in \mathbb{R}^d.
\end{align*}
\end{deff}
\begin{deff}
$f(x)$ has $\sigma$-bounded gradients if $\lVert \nabla f(x) \rVert \leq \sigma$ for all $x \in
\mathbb{R}^d$.
\end{deff}

In most real-world cases, the high-quality validation set is really small, and thus we could set the
mini-batch size $m$ to be the same as the size of the validation set $M$. Under this condition, the
following lemma shows that our algorithm always converges to a critical point of the validation
loss. However, our method is not equivalent to training a model only on this small validation set.
Because directly training a model on a small validation set will lead to severe overfitting issues.
On the contrary, our method can leverage useful information from a larger training set, and still
converge to an appropriate distribution favored by this clean and balanced validation dataset. This
helps both generalization and robustness to biases in the training set, which will be shown in our
experiments.

\begin{lem}\label{lem:convergence}
Suppose the validation loss function is Lipschitz-smooth with constant $L$, and the train loss
function $f_i$ of training data $x_i$ have $\sigma$-bounded gradients. Let the learning rate
$\alpha_t$ satisfies $\alpha_t \leq \frac{2n}{L\sigma^2}$, where $n$ is the training batch size.
Then, following our algorithm, the validation loss always monotonically decreases for any sequence of
training batches, namely,
\begin{align}
\label{eq:converge1}
G(\theta_{t+1}) \leq G(\theta_{t}),
\end{align}
where $G(\theta)$ is the total validation loss
\begin{align}
G(\theta) = \frac{1}{M} \sum_{i=1}^M f^v_i(\theta_{t+1}(\epsilon)).
\end{align}
Furthermore, in expectation, the equality in Eq. \ref{eq:converge1} holds only when the gradient of
validation loss becomes 0 at some time step $t$, namely $\EEsub{G(\theta_{t+1})}{t} = G(\theta_t)$
if and only if $\nabla G(\theta_t) = 0$, where the expectation is taking over possible training
batches at time step $t$.

\end{lem}
Moreover, we can prove the convergence rate of our method to be $O(1/\epsilon^2)$.

\begin{thm}\label{thm:convergencerate}
Suppose $G$, $f_i$ and $\alpha_t$ satisfy the aforementioned conditions, then
Algorithm~\ref{alg:ad} achieves $\EE{\lVert \nabla G(\theta_t) \rVert^2} \leq \epsilon$ in
$O(1/\epsilon^2)$ steps. More specifically,
\begin{align}
\min\limits_{0 < t < T} \EE{\lVert \nabla G(\theta_t) \rVert^2} \leq \frac{C}{\sqrt{T}},
\end{align}
where $C$ is some constant independent of the convergence process.

\end{thm}


\section{Experiments}

To test the effectiveness of our reweighting algorithm, we designed both class imbalance and noisy
label settings, and a combination of both, on standard MNIST and CIFAR benchmarks for image
classification using deep CNNs.
\footnote{Code released at: \url{https://github.com/uber-research/learning-to-reweight-examples}}

\subsection{MNIST data imbalance experiments} 

We use the standard MNIST handwritten digit classification dataset and subsample the dataset to
generate a class imbalance binary classification  task. We select a total of 5,000 images of size
28$\times$28 on class 4 and 9, where 9 dominates the training data distribution. We train a standard
LeNet on this task and we compare our method with a suite of commonly used tricks for class
imbalance: 1) \textsc{Proportion} weights each example by the inverse frequency 2) \textsc{Resample}
samples a class-balanced mini-batch for each iteration 3) \textsc{Hard Mining} selects the highest
loss examples from the majority class and 4) \textsc{Random} is a random example weight baseline
that assigns weights based on a rectified Gaussian distribution:
\begin{equation}
w_i^{\text{rnd}} = \frac{\max(z_i, 0)}{\sum_i \max(z_i, 0)}, \ \ \ \text{where} \ z_i \sim \mathcal{N}(0, 1).
\label{eq:randomwts}
\end{equation}
To make sure that our method does
not have the privilege of training on more data, we split the balanced validation set of 10 images
directly from the training set. The network is trained with SGD with a learning rate of 1e-3 and
mini-batch size of 100 for a total of 8,000 steps.

Figure~\ref{fig:mnist_imba} plots the test error rate across various imbalance ratios averaged from
10 runs with random splits. Note that our method significantly outperforms all the baselines. With
class imbalance ratio of 200:1, our method only reports a small increase of error rate around 2\%,
whereas other methods suffer terribly under this setting. Compared with resampling and hard negative
mining baselines, our approach does not throw away samples based on its class or training loss - as
long as a sample is helpful towards the validation loss, it will be included as a part of the
training loss.


\begin{figure}
\centering
\includegraphics[width=0.9\columnwidth]{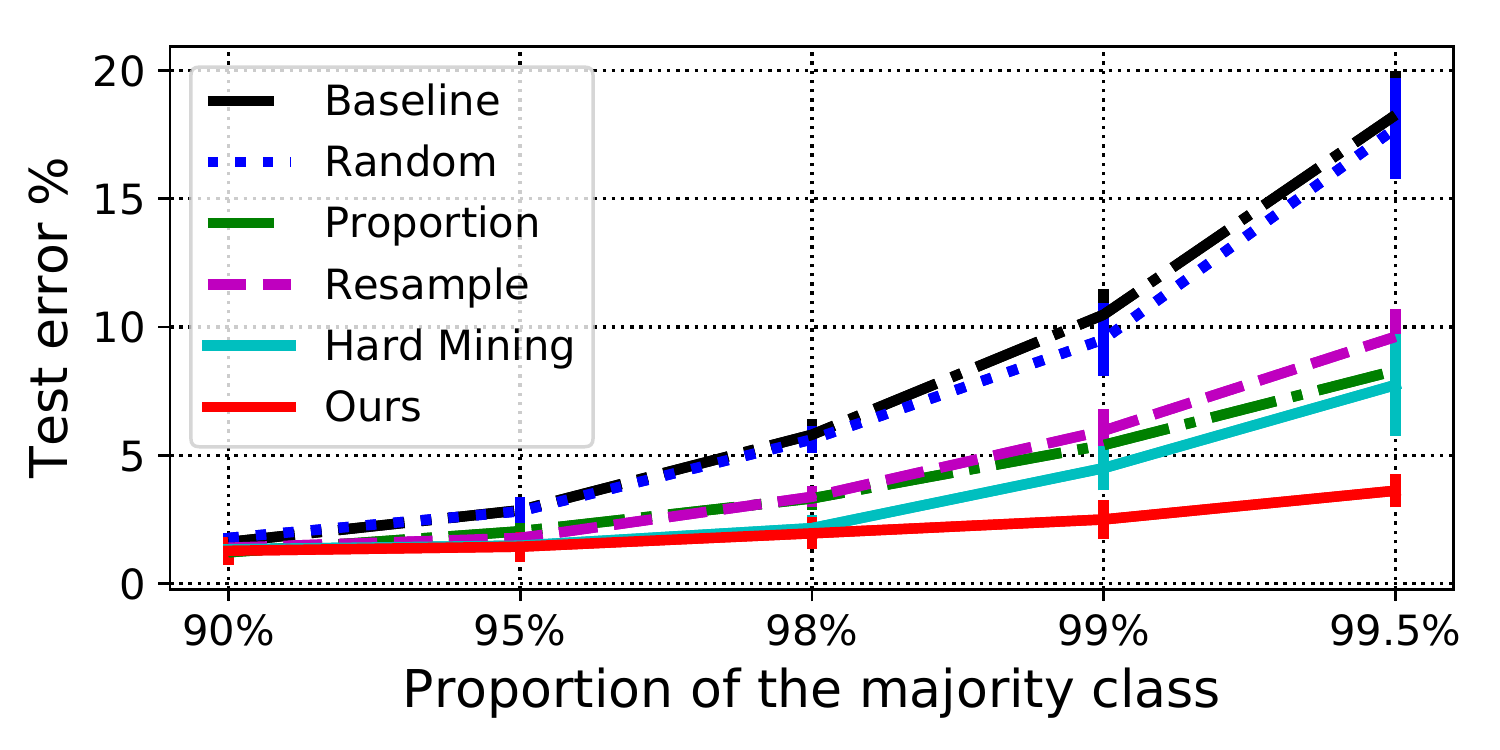}
\vspace{-0.1in}
\caption{MNIST 4-9 binary classification error using a LeNet on imbalanced classes. Our method uses a small balanced validation split of 10 examples.}
\label{fig:mnist_imba}
\end{figure}
\subsection{CIFAR noisy label experiments}

Reweighting algorithm can also be useful on datasets where the labels are noisy. We study two
settings of label noise here: 
\begin{itemize}
    \vspace{-0.1in}

    \item \textsc{UniformFlip}: All label classes can uniformly flip to any other label classes,
which is the most studied in the literature.

    \vspace{-0.1in}

    \item \textsc{BackgroundFlip}: All label classes can flip to a single background class. This
noise setting is very realistic. For instance, human annotators may not have recognized all the
positive instances, while the rest remain in the background class. This is also a combination of
label imbalance and label noise since the background class usually dominates the label distribution.

\end{itemize}
We compare our method with prior work on the noisy label problem.
\vspace{-0.1in}
\begin{itemize}
\item \textsc{Reed}, proposed by \citet{reed14noisy}, is a bootstrapping technique where the
training target is a convex combination of the model prediction and the label.

\item \textsc{S-Model}, proposed by \citet{goldberger17noise}, adds a fully connected softmax layer
after the regular classification output layer to model the noise transition matrix.

\item \textsc{MentorNet}, proposed by \citet{jiang17mentornet}, is an RNN-based meta-learning model
that takes in a sequence of loss values and outputs the example weights. We compare numbers reported
in their paper with a base model that achieves similar test accuracy under 0\% noise.
\end{itemize}
In addition, we propose two simple baselines: 
1) \textsc{Random}, which assigns weights according to a rectified Gaussian (see
   Eq.~\ref{eq:randomwts});
2) \textsc{Weighted}, designed for \textsc{BackgroundFlip}, where the model
knows the oracle noise ratio for each class and reweights the training loss proportional to the
percentage of clean images of that label class.
\vspace{-0.05in}
\paragraph{Clean validation set} For \textsc{UniformFlip}, we use 1,000 clean images in the
validation set; for \textsc{BackgroundFlip}, we use 10 clean images per label class. Since our
method uses information from the clean validation, for a fair comparison, we conduct an additional
finetuning on the clean data based on the pre-trained baselines. We also study the effect on the
size of the clean validation set in an ablation study.
\vspace{-0.05in}
\paragraph{Hyper-validation set} For monitoring training progress and tuning baseline
hyperparameters, we split out another 5,000 hyper-validation set from the 50,000 training images. We
also corrupt the hyper-validation set with the same noise type.

\begin{table}[t]
\begin{center}
\caption{CIFAR \textsc{UniformFlip} under 40\% noise ratio using a WideResNet-28-10 model. Test
accuracy shown in percentage. Top rows use only noisy data, and bottom uses additional 1000 clean
images. ``FT'' denotes fine-tuning on clean data.}
\label{tab:uniformflip}
\vskip 0.1in
\begin{small}
\begin{sc}
\begin{tabular}{ccc}
\toprule
Model              &  CIFAR-10                    & CIFAR-100                     \\
\midrule
Baseline           & 67.97 $\pm$ 0.62             & 50.66 $\pm$ 0.24              \\
Reed-Hard          & 69.66 $\pm$ 1.21             & 51.34 $\pm$  0.17             \\
S-Model            & 70.64 $\pm$ 3.09             & 49.10 $\pm$ 0.58              \\
MentorNet          & 76.6                         & 56.9                          \\
Random             & 86.06 $\pm$ 0.32.            & 58.01 $\pm$ 0.37              \\
\midrule
\multicolumn{3}{c}{Using 1,000 clean images} \\
\midrule
Clean Only         & 46.64 $\pm$ 3.90             & 9.94 $\pm$ 0.82               \\
Baseline +FT       & 78.66 $\pm$ 0.44             & 54.52 $\pm$ 0.40              \\
MentorNet +FT      & 78                           & 59                            \\
Random +FT         & 86.55 $\pm$ 0.24             & 58.54 $\pm$ 0.52              \\
Ours               & \textbf{86.92 $\pm$ 0.19}    & \textbf{61.34 $\pm$ 2.06}     \\
\bottomrule
\end{tabular}
\end{sc}
\end{small}
\end{center}
\vskip -0.1in
\end{table}

\begin{table}[t]
\begin{center}

\caption{CIFAR \textsc{BackgroundFlip} under 40\% noise ratio using a ResNet-32 model. Test accuracy
shown in percentage. Top rows use only noisy data, and bottom rows use additional 10 clean images
per class. ``+ES'' denotes early stopping; ``FT'' denotes fine-tuning.}

\label{tab:backgroundflip}

\resizebox{\columnwidth}{!}{
\begin{small}
\begin{sc}
\begin{tabular}{ccc}
\toprule
Model                             & CIFAR-10                    & CIFAR-100                     \\
\midrule
Baseline                          & 59.54 $\pm$ 2.16            & 37.82 $\pm$ 0.69              \\
Baseline +ES                      & 64.96 $\pm$ 1.19            & 39.08 $\pm$ 0.65              \\
Random                            & 69.51 $\pm$ 1.36            & 36.56 $\pm$ 0.44              \\
Weighted                          & 79.17 $\pm$ 1.36            & 36.56 $\pm$ 0.44              \\
Reed Soft +ES                     & 63.47 $\pm$ 1.05            & 38.44 $\pm$ 0.90              \\
Reed Hard +ES                     & 65.22 $\pm$ 1.06            & 39.03 $\pm$ 0.55              \\
S-Model                           & 58.60 $\pm$ 2.33            & 37.02 $\pm$ 0.34              \\
S-Model +Conf                     & 68.93 $\pm$ 1.09            & 46.72 $\pm$ 1.87              \\
S-Model +Conf +ES                 & 79.24 $\pm$ 0.56            & 54.50 $\pm$ 2.51              \\
\midrule
\multicolumn{3}{c}{Using 10 clean images per class} \\
\midrule
Clean Only                        & 15.90 $\pm$ 3.32            & 8.06  $\pm$ 0.76              \\
Baseline +FT                      & 82.82 $\pm$ 0.93            & 54.23 $\pm$ 1.75              \\
Baseline +ES +FT                  & 85.19 $\pm$ 0.46            & 55.22 $\pm$ 1.40              \\
Weighted +FT                      & 85.98 $\pm$ 0.47            & 53.99 $\pm$ 1.62              \\
S-Model +Conf +FT                 & 81.90 $\pm$ 0.85            & 53.11 $\pm$ 1.33              \\
S-Model +Conf +ES +FT             & 85.86 $\pm$ 0.63            & 55.75 $\pm$ 1.26              \\
Ours                              & \textbf{86.73 $\pm$ 0.48}   & \textbf{59.30 $\pm$ 0.60}     \\
\bottomrule
\end{tabular}
\end{sc}
\end{small}
}
\end{center}
\vskip -0.1in
\end{table}
\vspace{-0.05in}
\paragraph{Experimental details} For \textsc{Reed} model, we use the best $\beta$ reported in
\citet{reed14noisy} ($\beta=0.8$ for hard bootstrapping and $\beta=0.95$ for soft bootstrapping).
For the \textsc{S-Model}, we explore two versions to initialize the transition weights: 1) a
smoothed identity matrix; 2) in background flip experiments we consider initializing the transition
matrix with the confusion matrix of a pre-trained baseline model (\textsc{S-Model +Conf}). We find
baselines can easily overfit the training noise, and therefore we also study early stopped versions
of the baselines to provide a stronger comparison. In contrast, we find early stopping not necessary
for our method.

To make our results comparable with the ones reported in \textsc{MentorNet} and to save computation
time, we exchange their Wide ResNet-101-10 with a Wide ResNet-28-10 (WRN-28-10) \cite{wrn} with
dropout 0.3 as our base model in the \textsc{UniformFlip} experiments. We find that test accuracy
differences between the two base models are within 0.5\% on CIFAR datasets under 0\% noise. In the
\textsc{BackgroundFlip} experiments, we use a ResNet-32 \cite{resnet} as our base model.

We train the models with SGD with momentum, at an initial learning rate 0.1 and a momentum 0.9 with
mini-batch size 100. For ResNet-32 models, the learning rate decays $\times 0.1$ at 40K and 60K
steps, for a total of 80K steps. For WRN and early stopped versions of ResNet-32 models, the
learning rate decays at 40K and 50K steps, for a total of 60K steps. Under regular 0\% noise
settings, our base ResNet-32 gets 92.5\% and 68.1\% classification accuracy on CIFAR-10 and 100, and
the WRN-28-10 gets 95.5\% and 78.2\%. For the finetuning stage, we run extra 5K steps of training on
the limited clean data.

We report the average test accuracy for 5 different random splits of clean and noisy labels, with
95\% confidence interval in Table~\ref{tab:uniformflip} and \ref{tab:backgroundflip}. The background
classes for the 5 trials are [0, 1, 3, 5, 7] (CIFAR-10) and [7, 12, 41, 62, 85] (CIFAR-100).

\begin{figure}[t]
\centering
\includegraphics[width=0.45\columnwidth,trim={0cm 0 0cm 0},clip]{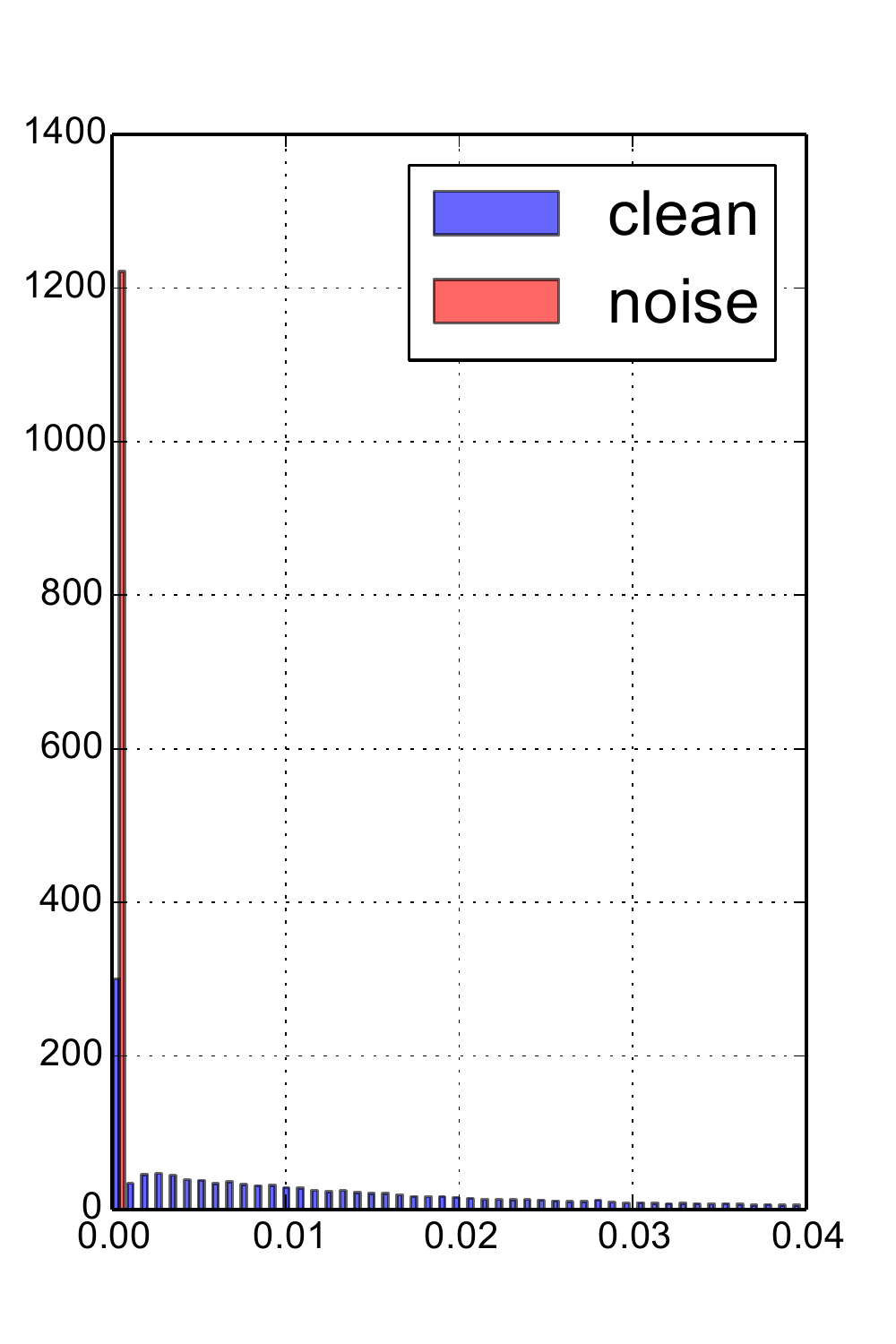}
\includegraphics[width=0.45\columnwidth,trim={0cm 0 0cm 0},clip]{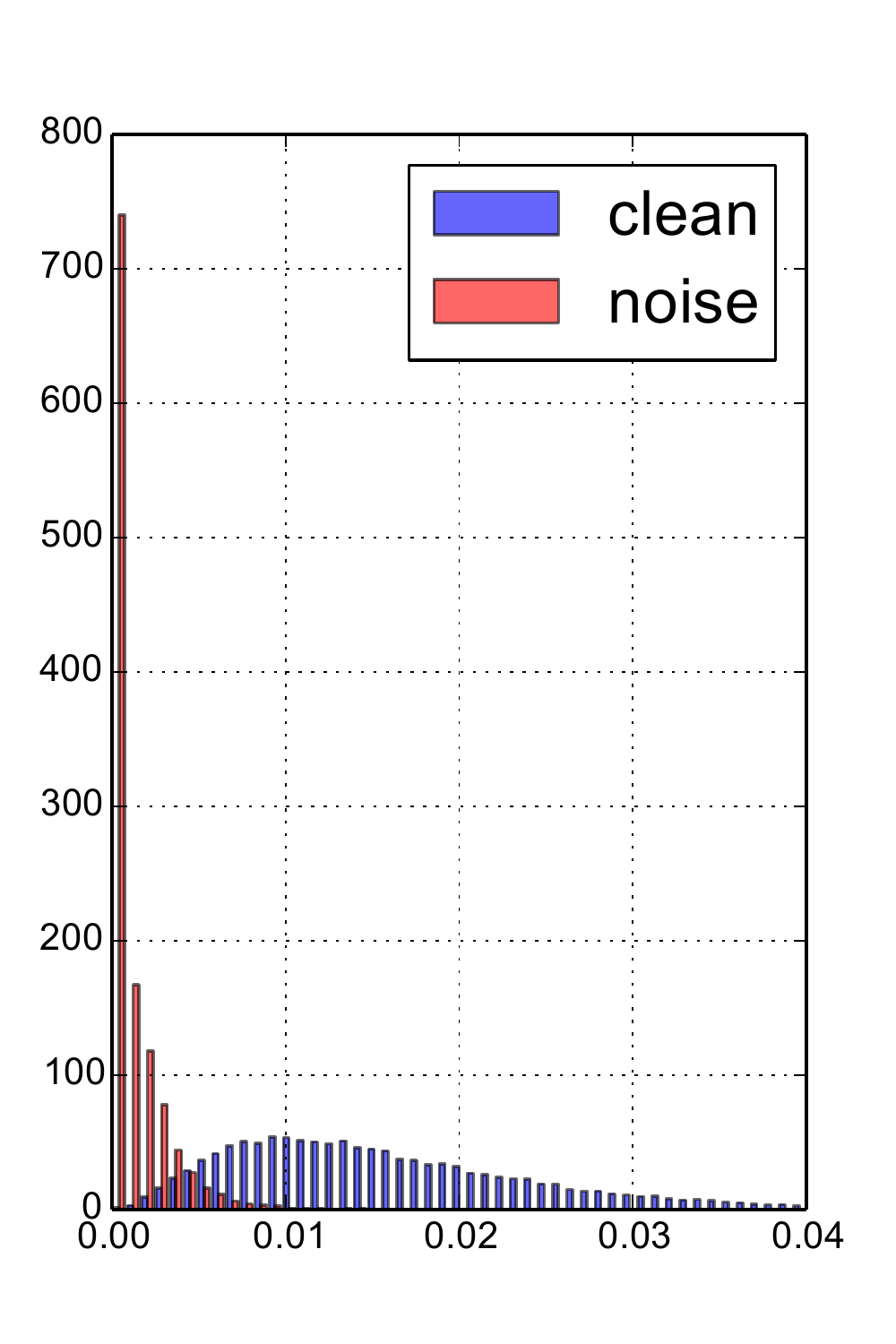}
\vspace{-0.1in}
\caption{Example weights distribution on \textsc{BackgroundFlip}. Left: a hyper-validation batch,
with randomly flipped background noises. Right: a hyper-validation batch containing only on a single
label class, with flipped background noises, averaged across all non-background classes.}
\label{fig:dist}
\end{figure}


\begin{figure}[t]
\centering
\includegraphics[width=0.9\columnwidth]{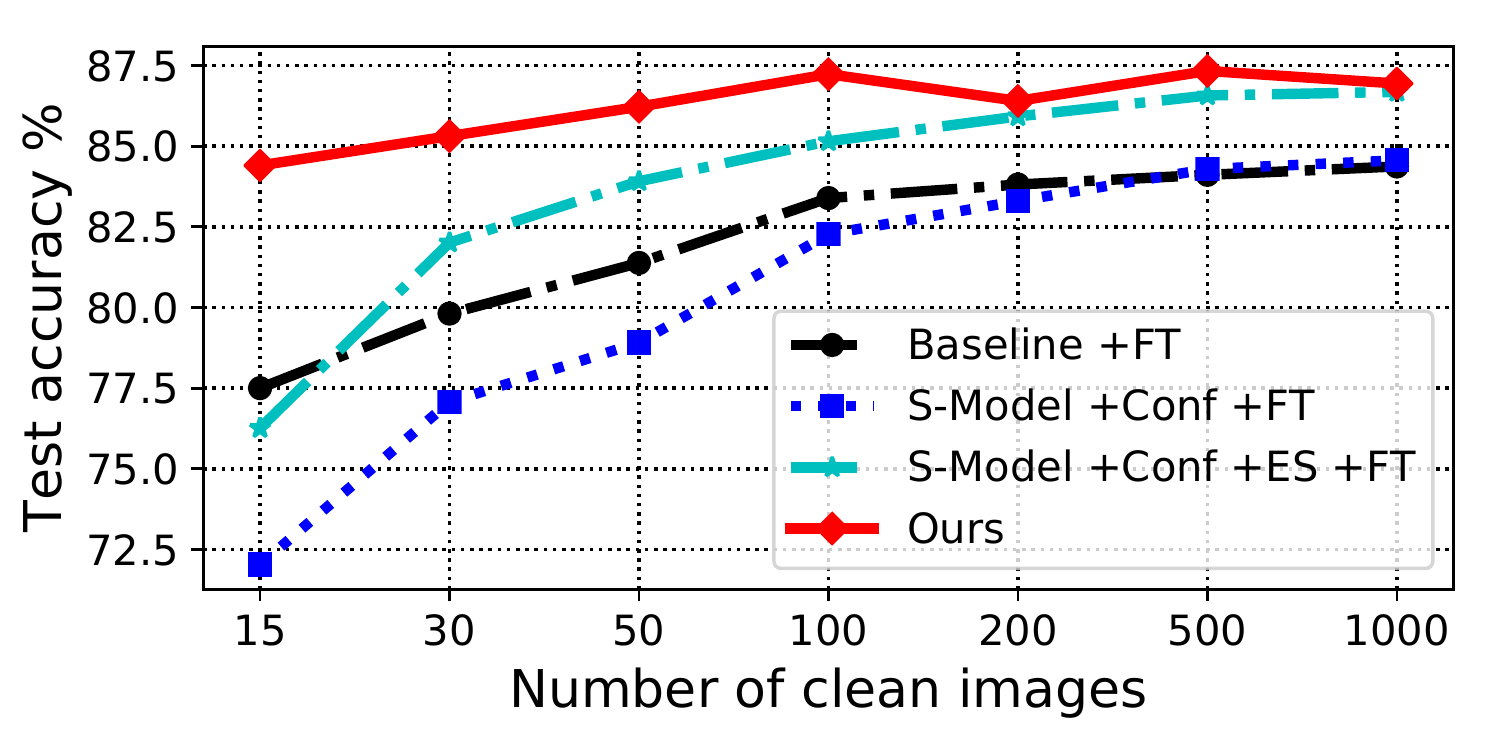}
\vspace{-0.1in}
\caption{Effect of the number of clean imaged used, on CIFAR-10 with 40\% of data flipped to label 3. ``ES'' denotes early stopping.}
\label{fig:ft}
\end{figure}

\begin{figure}[t]
\centering
\includegraphics[width=0.9\columnwidth]{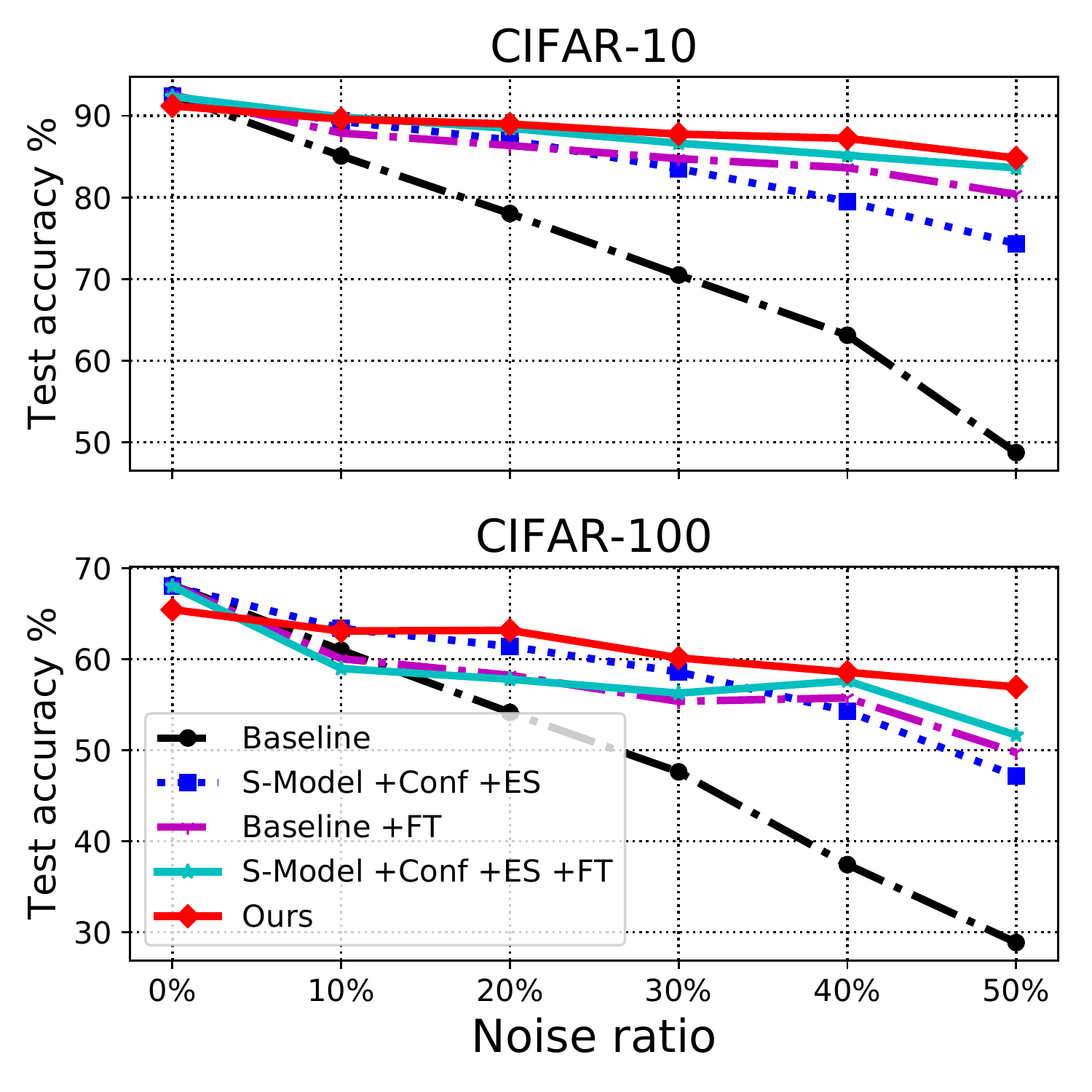}
\vspace{-0.1in}

\caption{Model test accuracy on imbalanced noisy CIFAR experiments across various noise levels using
a base ResNet-32 model. ``ES'' denotes early stopping, and ``FT'' denotes finetuning.}
\label{fig:level}

\end{figure}
\begin{figure}[h!]
\centering
\vspace{-0.1in}
\includegraphics[width=\columnwidth,trim={2.5cm 0 4cm 0},clip]{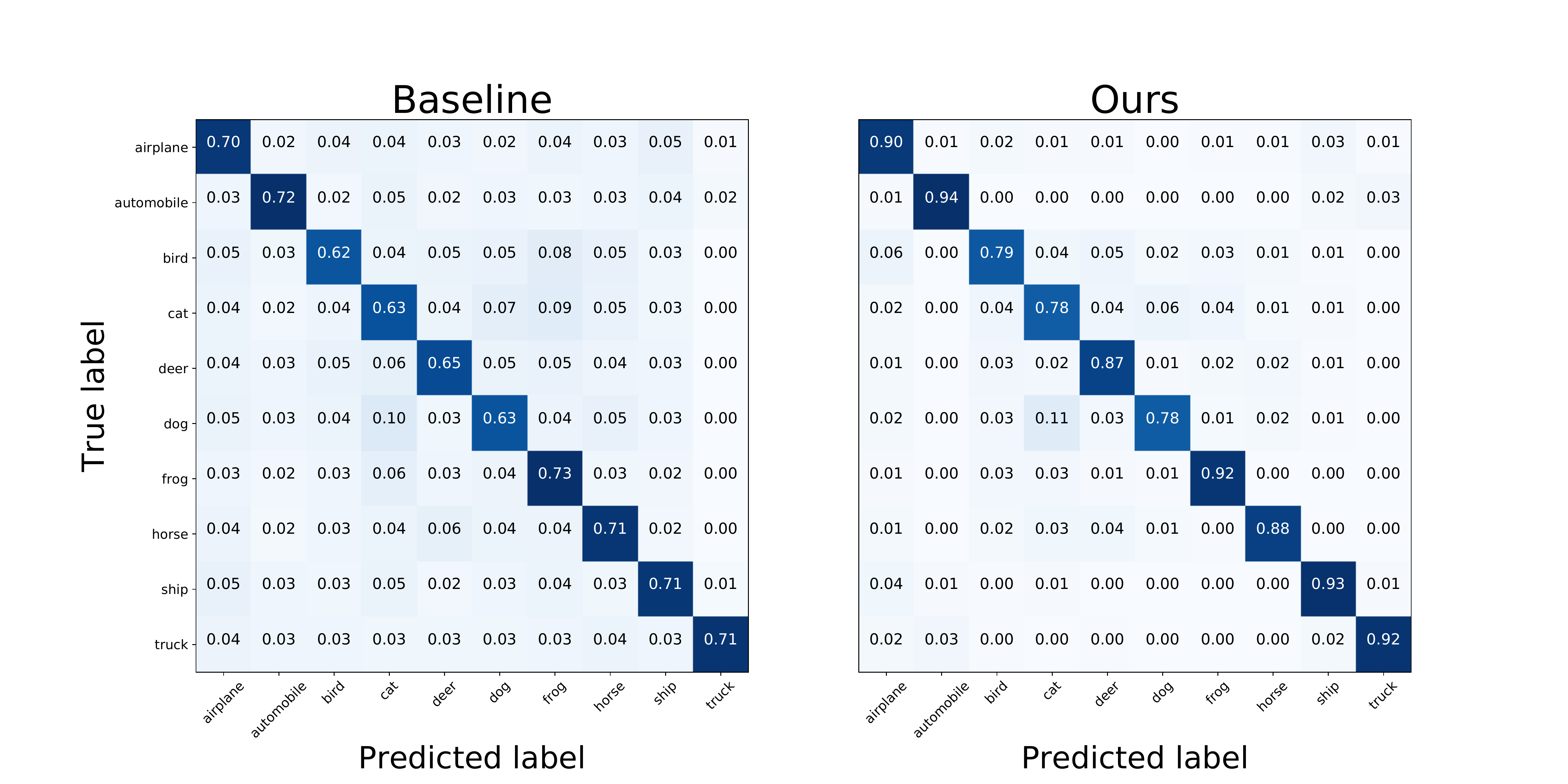}
\includegraphics[width=\columnwidth,trim={2.5cm 0 4cm 0},clip]{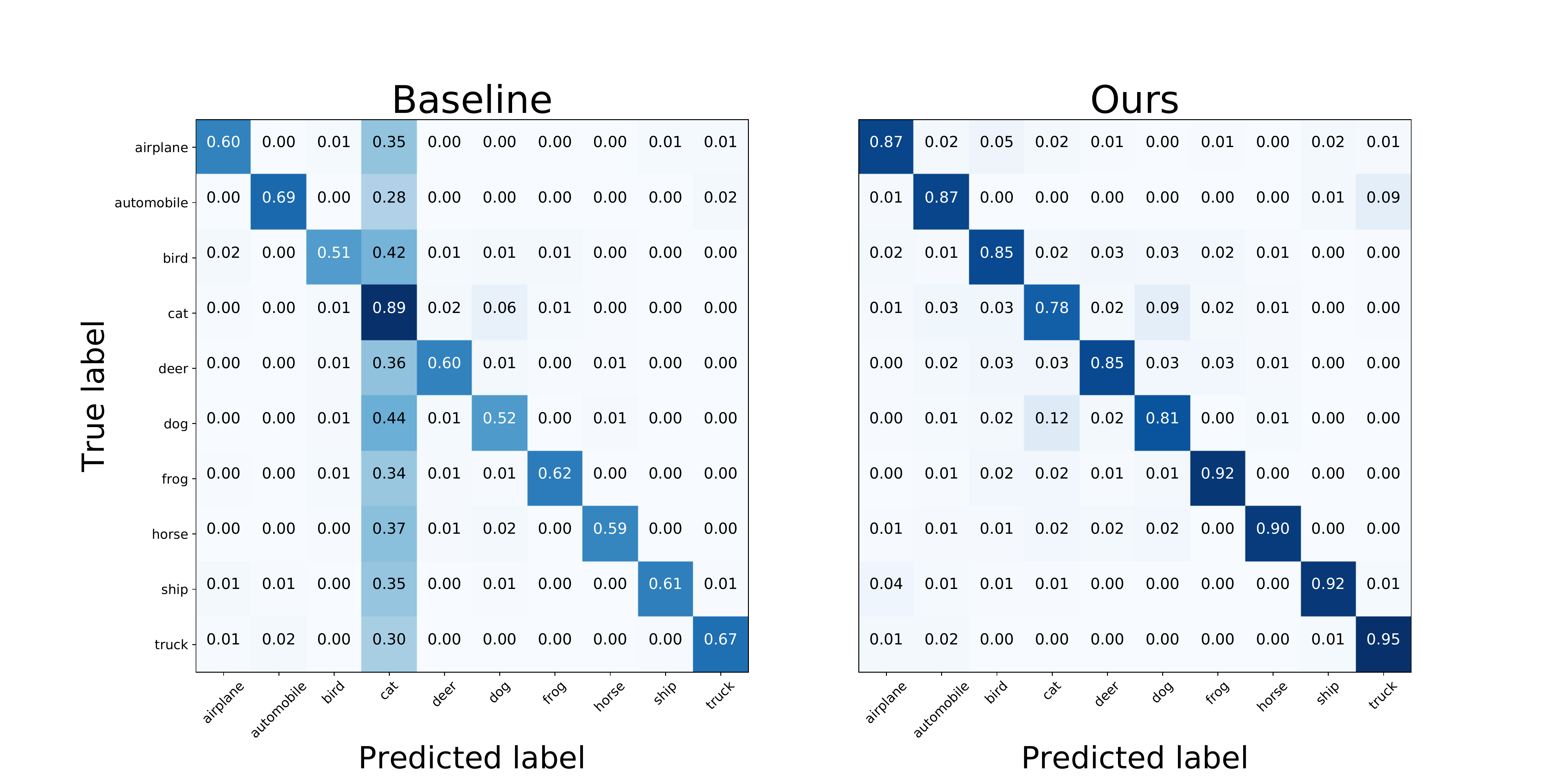}
\vspace{-0.2in}
\caption{Confusion matrices on CIFAR-10 \textsc{UniformFlip} (top) and \textsc{BackgroundFlip} (bottom)}
\label{fig:confusion}
\vspace{-0.1in}
\end{figure}
\begin{figure}[h]
\centering
\includegraphics[width=0.9\columnwidth,trim={0cm 0cm 0cm 0cm},clip]{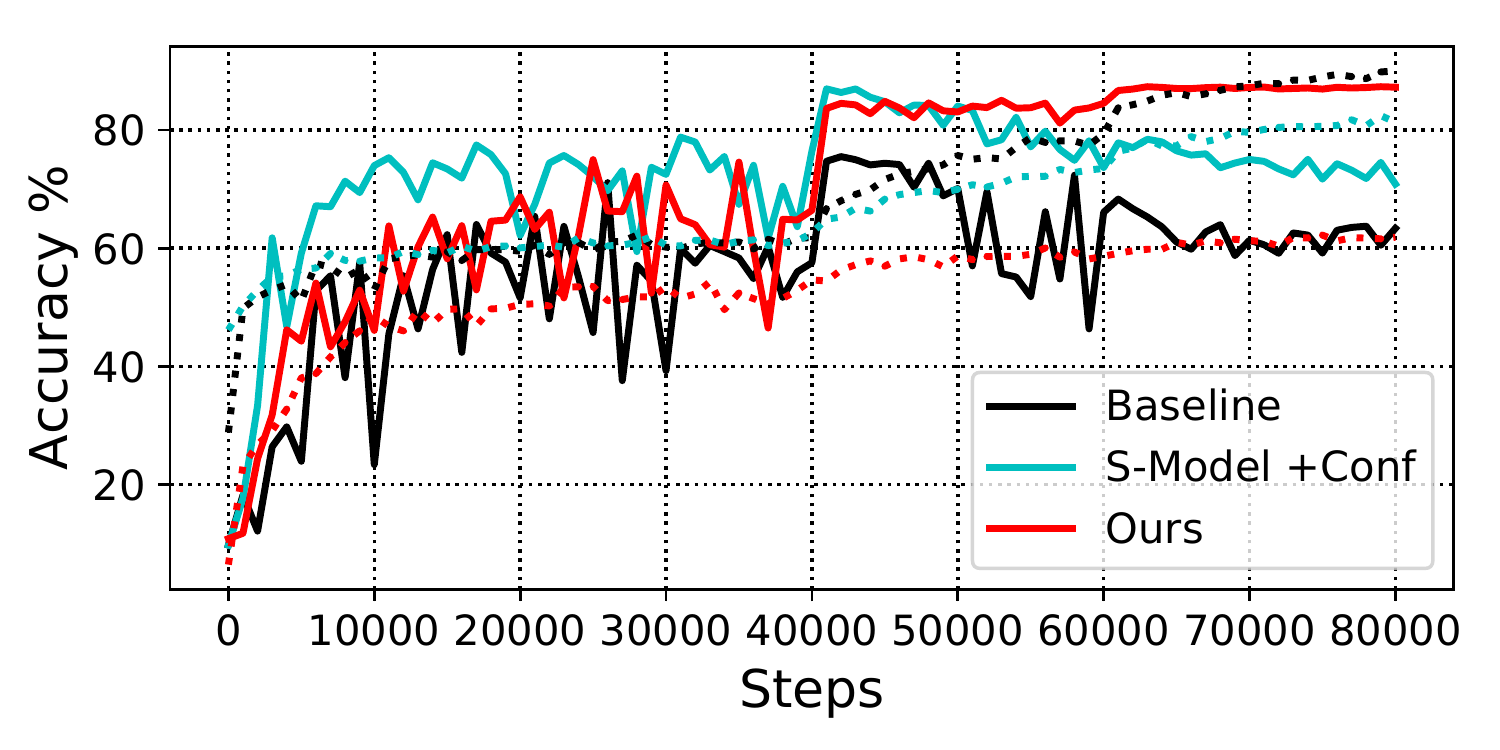}
\vspace{-0.1in}
\caption{Training curve of a ResNet-32 on CIFAR-10 \textsc{BackgroundFlip} under 40\% noise ratio.
Solid lines denote validation accuracy and dotted lines denote training. Our method is less prone to
label noise overfitting.}
\label{fig:curve}
\vspace{-0.15in}
\end{figure}

\subsection{Results and Discussion}

The first result that draws our attention is that ``Random'' performs surprisingly well on the
\textsc{UniformFlip} benchmark, outperforming all historical methods that we compared. Given that
its performance is comparable with Baseline on \textsc{BackgroundFlip} and MNIST class imbalance, we
hypothesize that random example weights act as a strong regularizer and under which the learning
objective on \textsc{UniformFlip} is still consistent.

Regardless of the strong baseline, our method ranks the top on both \textsc{UniformFlip} and
\textsc{BackgroundFlip}, showing our method is less affected by the changes in the noise type. On
CIFAR-100, our method wins more than 3\% compared to the state-of-the-art method.
\vspace{-0.05in}
\paragraph{Understanding the reweighting mechanism}
It is beneficial to understand how our reweighting algorithm contributes to learning more robust
models during training. First, we use a pre-trained model (trained at half of the total iterations
without learning rate decay) and measure the example weight distribution of a randomly sampled batch
of validation images, which the model has never seen. As shown in the left figure of Figure
\ref{fig:dist}, our model correctly pushes most noisy images to zero weights. Secondly,  we
conditioned the input mini-batch to be a single non-background class and randomly flip 40\% of the
images to the background, and we would like to see how well our model can distinguish clean and
noisy images. As shown in Figure \ref{fig:dist} right, the model is able to reliably detect images
that are flipped to the background class.
\vspace{-0.05in}
\paragraph{Robustness to overfitting noise} Throughout experimentation, we find baseline
models can easily overfit to the noise in the training set. For example, shown in
Table~\ref{tab:backgroundflip}, applying early stopping (``ES'') helps the classification
performance of ``S-Model'' by over 10\% on CIFAR-10. Figure~\ref{fig:confusion} compares the final
confusion matrices of the baseline and the proposed algorithm, where a large proportion of noise
transition probability is cleared in the final prediction. Figure~\ref{fig:curve} shows training
curves on the \textsc{BackgroundFlip} experiments. After the first learning rate decay, both
``Baseline'' and ``S-Model'' quickly degrade their validation performance due to overfitting, while
our model remains the same validation accuracy until termination. Note that here ``S-Model'' knows
the oracle noise ratio in each class, and this information is not available in our method.
\vspace{-0.05in}
\paragraph{Impact of the noise level} We would like to investigate how strongly our method can
perform on a variety of noise levels. Shown in Figure~\ref{fig:level}, our method only drops 6\%
accuracy when the noise ratio increased from 0\% to 50\%; whereas the baseline has dropped more than
40\%. At 0\% noise, our method only slightly underperforms baseline. This is reasonable since we are
optimizing on the validation set, which is strictly a subset of the full training set, and therefore
suffers from its own subsample bias.
\vspace{-0.05in}
\paragraph{Size of the clean validation set} When the size of the clean validation set grows larger,
fine-tuning on the validation set will be a reasonble approach. Here, we make an attempt to explore
the tradeoff and understand when fine-tuning becomes beneficial. Figure~\ref{fig:ft} plots the
classification performance when we varied the size of the clean validation on
\textsc{BackgroundFlip}. Surprisingly, using 15 validation images for all classes only results in a
2\% drop in performance, and the overall classification performance does not grow after having more
than 100 validation images. In comparison, we observe a significant drop in performance when only
fine-tuning on these 15 validation images for the baselines, and the performance catches up around
using 1,000 validation images (100 per class). This phenomenon suggests that in our method the clean
validation acts more like a regularizer rather than a data source for parameter fine-tuning, and
potentially our method can be complementary with fine-tuning based method when the size of the clean
set grows larger.


\section{Conclusion}
In this work, we propose an online meta-learning algorithm for reweighting training examples and
training more robust deep learning models. While various types of training set biases exist and
manually designed reweighting objectives have their own bias, our automatic reweighting algorithm
shows superior performance dealing with class imbalance, noisy labels, and both. Our method can be
directly applied to any deep learning architecture and is expected to train end-to-end without any
additional hyperparameter search. Validating on every training step is a novel setting and we show
that it has links with model regularization, which can be a fruitful future research direction.

\let\oldbibliography\thebibliography
\renewcommand{\thebibliography}[1]{\oldbibliography{#1}
\setlength{\itemsep}{7pt}} 

\bibliography{our_ref}

\begin{thebibliography}{44}
\providecommand{\natexlab}[1]{#1}
\providecommand{\url}[1]{\texttt{#1}}
\expandafter\ifx\csname urlstyle\endcsname\relax
  \providecommand{\doi}[1]{doi: #1}\else
  \providecommand{\doi}{doi: \begingroup \urlstyle{rm}\Url}\fi

\bibitem[Abadi et~al.(2016)Abadi, Barham, Chen, Chen, Davis, Dean, Devin,
  Ghemawat, Irving, Isard, Kudlur, Levenberg, Monga, Moore, Murray, Steiner,
  Tucker, Vasudevan, Warden, Wicke, Yu, and Zheng]{tensorflow}
Abadi, Mart{\'{\i}}n, Barham, Paul, Chen, Jianmin, Chen, Zhifeng, Davis, Andy,
  Dean, Jeffrey, Devin, Matthieu, Ghemawat, Sanjay, Irving, Geoffrey, Isard,
  Michael, Kudlur, Manjunath, Levenberg, Josh, Monga, Rajat, Moore, Sherry,
  Murray, Derek~Gordon, Steiner, Benoit, Tucker, Paul~A., Vasudevan, Vijay,
  Warden, Pete, Wicke, Martin, Yu, Yuan, and Zheng, Xiaoqiang.
\newblock Tensorflow: {A} system for large-scale machine learning.
\newblock In \emph{12th {USENIX} Symposium on Operating Systems Design and
  Implementation, {OSDI}}, 2016.

\bibitem[Andrychowicz et~al.(2016)Andrychowicz, Denil, Colmenarejo, Hoffman,
  Pfau, Schaul, and de~Freitas]{l2l}
Andrychowicz, Marcin, Denil, Misha, Colmenarejo, Sergio~Gomez, Hoffman,
  Matthew~W., Pfau, David, Schaul, Tom, and de~Freitas, Nando.
\newblock Learning to learn by gradient descent by gradient descent.
\newblock In \emph{Advances in Neural Information Processing Systems, {NIPS}},
  2016.

\bibitem[Angluin \& Laird(1988)Angluin and Laird]{noisytheory}
Angluin, Dana and Laird, Philip.
\newblock Learning from noisy examples.
\newblock \emph{Machine Learning}, 2\penalty0 (4):\penalty0 343--370, Apr 1988.
\newblock ISSN 1573-0565.

\bibitem[Azadi et~al.(2016)Azadi, Feng, Jegelka, and Darrell]{azadi16air}
Azadi, Samaneh, Feng, Jiashi, Jegelka, Stefanie, and Darrell, Trevor.
\newblock Auxiliary image regularization for deep cnns with noisy labels.
\newblock In \emph{Proceedings of the 4th International Conference on Learning
  Representation, {ICLR}}, 2016.

\bibitem[Bengio et~al.(2009)Bengio, Louradour, Collobert, and
  Weston]{bengio09curriculum}
Bengio, Yoshua, Louradour, J{\'{e}}r{\^{o}}me, Collobert, Ronan, and Weston,
  Jason.
\newblock Curriculum learning.
\newblock In \emph{Proceedings of the 26th Annual International Conference on
  Machine Learning, {ICML}}, 2009.

\bibitem[Chang et~al.(2017)Chang, Learned{-}Miller, and
  McCallum]{chang17activebias}
Chang, Haw{-}Shiuan, Learned{-}Miller, Erik~G., and McCallum, Andrew.
\newblock Active bias: Training more accurate neural networks by emphasizing
  high variance samples.
\newblock In \emph{Advances in Neural Information Processing Systems, {NIPS}},
  2017.

\bibitem[Chawla et~al.(2002)Chawla, Bowyer, Hall, and Kegelmeyer]{smote}
Chawla, Nitesh~V., Bowyer, Kevin~W., Hall, Lawrence~O., and Kegelmeyer,
  W.~Philip.
\newblock {SMOTE:} synthetic minority over-sampling technique.
\newblock \emph{J. Artif. Intell. Res.}, 16:\penalty0 321--357, 2002.

\bibitem[Chen \& Gupta(2015)Chen and Gupta]{webly}
Chen, Xinlei and Gupta, Abhinav.
\newblock Webly supervised learning of convolutional networks.
\newblock In \emph{Proceedings of the 2015 IEEE International Conference on
  Computer Vision, {ICCV}}, 2015.

\bibitem[Cordts et~al.(2016)Cordts, Omran, Ramos, Rehfeld, Enzweiler, Benenson,
  Franke, Roth, and Schiele]{cityscapes}
Cordts, Marius, Omran, Mohamed, Ramos, Sebastian, Rehfeld, Timo, Enzweiler,
  Markus, Benenson, Rodrigo, Franke, Uwe, Roth, Stefan, and Schiele, Bernt.
\newblock The cityscapes dataset for semantic urban scene understanding.
\newblock In \emph{Proceedings of the IEEE Conference on Computer Vision and
  Pattern Recognition, {CVPR}}, 2016.

\bibitem[Dong et~al.(2017)Dong, Gong, and Zhu]{dong17imbalance}
Dong, Qi, Gong, Shaogang, and Zhu, Xiatian.
\newblock Class rectification hard mining for imbalanced deep learning.
\newblock In \emph{Proceedings of the {IEEE} International Conference on
  Computer Vision, {ICCV}}, 2017.

\bibitem[Finn et~al.(2017)Finn, Abbeel, and Levine]{maml}
Finn, Chelsea, Abbeel, Pieter, and Levine, Sergey.
\newblock Model-agnostic meta-learning for fast adaptation of deep networks.
\newblock In \emph{Proceedings of the 34th International Conference on Machine
  Learning, {ICML}}, 2017.

\bibitem[Freund \& Schapire(1997)Freund and Schapire]{adaboost}
Freund, Yoav and Schapire, Robert~E.
\newblock A decision-theoretic generalization of on-line learning and an
  application to boosting.
\newblock \emph{J. Comput. Syst. Sci.}, 55\penalty0 (1):\penalty0 119--139,
  1997.

\bibitem[Goldberger \& Ben-Reuven(2017)Goldberger and
  Ben-Reuven]{goldberger17noise}
Goldberger, Jacob and Ben-Reuven, Ehud.
\newblock Training deep neural-networks using a noise adaptation layer.
\newblock In \emph{Proceedings of the 5th International Conference on Learning
  Representation, {ICLR}}, 2017.

\bibitem[He et~al.(2016)He, Zhang, Ren, and Sun]{resnet}
He, Kaiming, Zhang, Xiangyu, Ren, Shaoqing, and Sun, Jian.
\newblock Deep residual learning for image recognition.
\newblock In \emph{Proceedings of the {IEEE} Conference on Computer Vision and
  Pattern Recognition, {CVPR}}, 2016.

\bibitem[Hendrycks et~al.(2018)Hendrycks, Mazeika, Wilson, and Gimpel]{glc}
Hendrycks, Dan, Mazeika, Mantas, Wilson, Duncan, and Gimpel, Kevin.
\newblock Using trusted data to train deep networks on labels corrupted by
  severe noise.
\newblock \emph{CoRR}, abs/1802.05300, 2018.

\bibitem[Huang et~al.(2016)Huang, Li, Loy, and Tang]{lmle}
Huang, Chen, Li, Yining, Loy, Chen~Change, and Tang, Xiaoou.
\newblock Learning deep representation for imbalanced classification.
\newblock In \emph{Proceedings of the {IEEE} Conference on Computer Vision and
  Pattern Recognition, {CVPR}}, 2016.

\bibitem[Jiang et~al.(2015)Jiang, Meng, Zhao, Shan, and Hauptmann]{spcl}
Jiang, Lu, Meng, Deyu, Zhao, Qian, Shan, Shiguang, and Hauptmann, Alexander~G.
\newblock Self-paced curriculum learning.
\newblock In \emph{Proceedings of the 29th {AAAI} Conference on Artificial
  Intelligence}, 2015.

\bibitem[Jiang et~al.(2017)Jiang, Zhou, Leung, Li, and
  Fei{-}Fei]{jiang17mentornet}
Jiang, Lu, Zhou, Zhengyuan, Leung, Thomas, Li, Li{-}Jia, and Fei{-}Fei, Li.
\newblock Mentornet: Regularizing very deep neural networks on corrupted
  labels.
\newblock \emph{CoRR}, abs/1712.05055, 2017.

\bibitem[Kahn \& Marshall(1953)Kahn and Marshall]{importantsample}
Kahn, Herman and Marshall, Andy~W.
\newblock Methods of reducing sample size in monte carlo computations.
\newblock \emph{Journal of the Operations Research Society of America},
  1\penalty0 (5):\penalty0 263--278, 1953.

\bibitem[Khan et~al.(2015)Khan, Bennamoun, Sohel, and
  Togneri]{costsensitivedeep}
Khan, Salman~Hameed, Bennamoun, Mohammed, Sohel, Ferdous~Ahmed, and Togneri,
  Roberto.
\newblock Cost sensitive learning of deep feature representations from
  imbalanced data.
\newblock \emph{CoRR}, abs/1508.03422, 2015.

\bibitem[Koh \& Liang(2017)Koh and Liang]{kohL17influence}
Koh, Pang~Wei and Liang, Percy.
\newblock Understanding black-box predictions via influence functions.
\newblock In \emph{Proceedings of the 34th International Conference on Machine
  Learning, {ICML}}, 2017.

\bibitem[Kumar et~al.(2010)Kumar, Packer, and Koller]{kumar10selfpaced}
Kumar, M.~Pawan, Packer, Benjamin, and Koller, Daphne.
\newblock Self-paced learning for latent variable models.
\newblock In \emph{Advances in Neural Information Processing Systems, {NIPS}},
  2010.

\bibitem[Lake et~al.(2017)Lake, Ullman, Tenenbaum, and Gershman]{lakemetalearn}
Lake, Brenden~M., Ullman, Tomer~D., Tenenbaum, Joshua~B., and Gershman,
  Samuel~J.
\newblock {B}uilding machines that learn and think like people.
\newblock \emph{Behav Brain Sci}, 40:\penalty0 e253, Jan 2017.

\bibitem[Li et~al.(2017)Li, Yang, Song, Cao, Luo, and Li]{li17noisydistill}
Li, Yuncheng, Yang, Jianchao, Song, Yale, Cao, Liangliang, Luo, Jiebo, and Li,
  Li{-}Jia.
\newblock Learning from noisy labels with distillation.
\newblock In \emph{Proceedings of the {IEEE} International Conference on
  Computer Vision, {ICCV}}, 2017.

\bibitem[Lin et~al.(2017)Lin, Goyal, Girshick, He, and Doll{\'{a}}r]{focal}
Lin, Tsung{-}Yi, Goyal, Priya, Girshick, Ross~B., He, Kaiming, and
  Doll{\'{a}}r, Piotr.
\newblock Focal loss for dense object detection.
\newblock In \emph{Proceedings of the {IEEE} International Conference on
  Computer Vision, {ICCV}}, 2017.

\bibitem[Lorraine \& Duvenaud(2018)Lorraine and Duvenaud]{hpernet}
Lorraine, Jonathan and Duvenaud, David.
\newblock Stochastic hyperparameter optimization through hypernetworks.
\newblock \emph{CoRR}, abs/1802.09419, 2018.

\bibitem[Ma et~al.(2017)Ma, Meng, Xie, Li, and Dong]{spaco}
Ma, Fan, Meng, Deyu, Xie, Qi, Li, Zina, and Dong, Xuanyi.
\newblock Self-paced co-training.
\newblock In \emph{Proceedings of the 34th International Conference on Machine
  Learning, {ICML}}, 2017.

\bibitem[Malisiewicz et~al.(2011)Malisiewicz, Gupta, and Efros]{hardneg}
Malisiewicz, Tomasz, Gupta, Abhinav, and Efros, Alexei~A.
\newblock Ensemble of exemplar-svms for object detection and beyond.
\newblock In \emph{Proceedings of the {IEEE} International Conference on
  Computer Vision, {ICCV}}, 2011.

\bibitem[Mu{\~{n}}oz{-}Gonz{\'{a}}lez et~al.(2017)Mu{\~{n}}oz{-}Gonz{\'{a}}lez,
  Biggio, Demontis, Paudice, Wongrassamee, Lupu, and Roli]{datapoison}
Mu{\~{n}}oz{-}Gonz{\'{a}}lez, Luis, Biggio, Battista, Demontis, Ambra, Paudice,
  Andrea, Wongrassamee, Vasin, Lupu, Emil~C., and Roli, Fabio.
\newblock Towards poisoning of deep learning algorithms with back-gradient
  optimization.
\newblock In \emph{Proceedings of the 10th {ACM} Workshop on Artificial
  Intelligence and Security, AISec@CCS}, 2017.

\bibitem[Natarajan et~al.(2013)Natarajan, Dhillon, Ravikumar, and
  Tewari]{natarajan13noisy}
Natarajan, Nagarajan, Dhillon, Inderjit~S., Ravikumar, Pradeep, and Tewari,
  Ambuj.
\newblock Learning with noisy labels.
\newblock In \emph{Advances in Neural Information Processing Systems, {NIPS}},
  2013.

\bibitem[Ravi \& Larochelle(2017)Ravi and Larochelle]{ravi2017oneshot}
Ravi, Sachin and Larochelle, Hugo.
\newblock Optimization as a model for few-shot learning.
\newblock In \emph{Proceedings of the 5th International Conference on Learning
  Representations, {ICLR}}, 2017.

\bibitem[Reddi et~al.(2016)Reddi, Hefny, Sra, P{\'{o}}czos, and Smola]{svrg}
Reddi, Sashank~J., Hefny, Ahmed, Sra, Suvrit, P{\'{o}}czos, Barnab{\'{a}}s, and
  Smola, Alexander~J.
\newblock Stochastic variance reduction for nonconvex optimization.
\newblock In \emph{Proceedings of the 33rd International Conference on Machine
  Learning, {ICML}}, 2016.

\bibitem[Reed et~al.(2014)Reed, Lee, Anguelov, Szegedy, Erhan, and
  Rabinovich]{reed14noisy}
Reed, Scott~E., Lee, Honglak, Anguelov, Dragomir, Szegedy, Christian, Erhan,
  Dumitru, and Rabinovich, Andrew.
\newblock Training deep neural networks on noisy labels with bootstrapping.
\newblock \emph{CoRR}, abs/1412.6596, 2014.

\bibitem[Ren et~al.(2018)Ren, Triantafillou, Ravi, Snell, Swersky, Tenenbaum,
  Larochelle, and Zemel]{metafewshot}
Ren, Mengye, Triantafillou, Eleni, Ravi, Sachin, Snell, Jake, Swersky, Kevin,
  Tenenbaum, Joshua~B., Larochelle, Hugo, and Zemel, Richard~S.
\newblock Meta learning for few-shot semi-supervised classification.
\newblock In \emph{Proceedings of the 6th International Conference on Learning
  Representations, {ICLR}}, 2018.

\bibitem[Russakovsky et~al.(2015)Russakovsky, Deng, Su, Krause, Satheesh, Ma,
  Huang, Karpathy, Khosla, Bernstein, Berg, and Fei-Fei]{ILSVRC15}
Russakovsky, Olga, Deng, Jia, Su, Hao, Krause, Jonathan, Satheesh, Sanjeev, Ma,
  Sean, Huang, Zhiheng, Karpathy, Andrej, Khosla, Aditya, Bernstein, Michael,
  Berg, Alexander~C., and Fei-Fei, Li.
\newblock {ImageNet Large Scale Visual Recognition Challenge}.
\newblock \emph{International Journal of Computer Vision, {IJCV}}, 115\penalty0
  (3):\penalty0 211--252, 2015.

\bibitem[Sukhbaatar \& Fergus(2014)Sukhbaatar and
  Fergus]{sukhbaatar14convnoise}
Sukhbaatar, Sainbayar and Fergus, Rob.
\newblock Learning from noisy labels with deep neural networks.
\newblock \emph{CoRR}, abs/1406.2080, 2014.

\bibitem[Thrun \& Pratt(1998)Thrun and Pratt]{metalearn}
Thrun, Sebastian and Pratt, Lorien.
\newblock \emph{Learning to Learn}.
\newblock Springer, 1998.

\bibitem[Ting(2000)]{costsensitive}
Ting, Kai~Ming.
\newblock A comparative study of cost-sensitive boosting algorithms.
\newblock In \emph{Proceedings of the 17th International Conference on Machine
  Learning, {ICML}}, 2000.

\bibitem[Vahdat(2017)]{vahdat17crf}
Vahdat, Arash.
\newblock Toward robustness against label noise in training deep discriminative
  neural networks.
\newblock In \emph{Advances in Neural Information Processing Systems, {NIPS}},
  2017.

\bibitem[Wang et~al.(2017)Wang, Kucukelbir, and Blei]{wang17reweight}
Wang, Yixin, Kucukelbir, Alp, and Blei, David~M.
\newblock Robust probabilistic modeling with bayesian data reweighting.
\newblock In \emph{Proceedings of the 34th International Conference on Machine
  Learning, {ICML}}, 2017.

\bibitem[Wu et~al.(2018)Wu, Ren, Liao, and Grosse]{shorthorizon}
Wu, Yuhuai, Ren, Mengye, Liao, Renjie, and Grosse, Roger~B.
\newblock Understanding short-horizon bias in stochastic meta-optimization.
\newblock In \emph{Proceedings of the 6th International Conference on Learning
  Representations, {ICLR}}, 2018.

\bibitem[Xiao et~al.(2015)Xiao, Xia, Yang, Huang, and Wang]{xiao15noisy}
Xiao, Tong, Xia, Tian, Yang, Yi, Huang, Chang, and Wang, Xiaogang.
\newblock Learning from massive noisy labeled data for image classification.
\newblock In \emph{Proceedings of the {IEEE} Conference on Computer Vision and
  Pattern Recognition, {CVPR}}, 2015.

\bibitem[Zagoruyko \& Komodakis(2016)Zagoruyko and Komodakis]{wrn}
Zagoruyko, Sergey and Komodakis, Nikos.
\newblock Wide residual networks.
\newblock In \emph{Proceedings of the British Machine Vision Conference,
  {BMVC}}, 2016.

\bibitem[Zhang et~al.(2017)Zhang, Bengio, Hardt, Recht, and Vinyals]{rethink}
Zhang, Chiyuan, Bengio, Samy, Hardt, Moritz, Recht, Benjamin, and Vinyals,
  Oriol.
\newblock Understanding deep learning requires rethinking generalization.
\newblock In \emph{Proceedings of the 5th International Conference on Learning
  Representations, {ICLR}}, 2017.

\end{thebibliography}
\bibliographystyle{icml2018}

\if\arxiv1
\newpage
\appendix

\section{Reweighting in an MLP}
\label{sec:mlp_derive}
We show the complete derivation below on calculating the example weights in an MLP network.
\begin{align}
&\frac{\partial}{\partial \epsilon_{i,t}} \EE{f^v(\theta_{t+1}(\epsilon))}
\Bigr|_{\epsilon_{i,t}=0}\\
=&\frac{1}{m}\sum_{j=1}^m \frac{\partial}{\partial \epsilon_{i,t}} f_j^v(\theta_{t+1}(\epsilon))
\Bigr|_{\epsilon_{i,t}=0}\\
=&\frac{1}{m}\sum_{j=1}^m \frac{\partial f_j^v(\theta)}{\partial \theta}\Bigr|_{\theta=\theta_t}^\top
\frac{\partial \theta_{t+1}(\epsilon_{i,t})} {\partial \epsilon_{i,t}}\Bigr|_{\epsilon_{i,t}=0}\\
\propto&-\frac{1}{m}\sum_{j=1}^m \frac{\partial f_j^v(\theta)}{\partial \theta}\Bigr|_{\theta=\theta_t}^\top
\frac{\partial f_i(\theta)}{\partial \theta}\Bigr|_{\theta=\theta_t}\\
=&-\frac{1}{m}\sum_{j=1}^m \sum_{l=1}^L
\frac{\partial f_j^v}{\partial \theta_l}\Bigr|_{\theta_l=\theta_{l,t}}^\top
\frac{\partial f_i}{\partial \theta_{l}}\Bigr|_{\theta_l=\theta_{l,t}}\\
=&-\frac{1}{m}\sum_{j=1}^m \sum_{l=1}^L
\vecc \left( \tilde{z}_{j,l-1}^v {g_{j,l}^v}^\top \right)^\top
\vecc \left( \tilde{z}_{i,l-1} g_{i,l}^\top \right)\\
=&-\frac{1}{m}\sum_{j=1}^m \sum_{l=1}^L \sum_{p=1}^{D_1} \sum_{q=1}^{D_2}
\tilde{z}_{j,l-1,p}^v g_{j,l,q}^v \tilde{z}_{i,l-1,p} g_{i,l,q}\\
=&-\frac{1}{m}\sum_{j=1}^m \sum_{l=1}^L
\sum_{p=1}^{D_1} \tilde{z}_{j,l-1,p}^v \tilde{z}_{i,l-1,p}
\sum_{q=1}^{D_2} g_{j,l,q}^v g_{i,l,q}\\
=&-\frac{1}{m}\sum_{j=1}^m \sum_{l=1}^L
(\tilde{z}^v_{j,l-1}{}^\top
\tilde{z}_{i,l-1})
(g^v_{j,l}{}^\top g_{i,l}).
\end{align}

\section{Convergence of our method}
\label{sec:lemproof}
This section provides the proof for Lemma~\ref{lem:convergence}.

\begin{lem*}

Suppose the validation loss function is Lipschitz-smooth with constant $L$, and the train loss
function $f_i$ of training data $x_i$ have $\sigma$-bounded gradients. Let the learning rate
$\alpha_t$ satisfies $\alpha_t \leq \frac{2n}{L\sigma^2}$, where $n$ is the training batch size.
Then, following our algorithm, the validation loss always monotonically decreases for any sequence of
training batches, namely,
\begin{align}
\label{eq:converge2}
G(\theta_{t+1}) \leq G(\theta_{t}),
\end{align}
where $G(\theta)$ is the total validation loss
\begin{align}
G(\theta) = \frac{1}{M} \sum_{i=1}^M f^v_i(\theta_{t+1}(\epsilon)).
\end{align}
Furthermore, in expectation, the equality in Eq. \ref{eq:converge2} holds only when the gradient of
validation loss becomes 0 at some time step $t$, namely $\EEsub{G(\theta_{t+1})}{t} = G(\theta_t)$
if and only if $\nabla G(\theta_t) = 0$, where the expectation is taking over possible training
batches at time step $t$.

\end{lem*}

\begin{proof}
Suppose we have a small validation set with $M$ clean data $\{x_1, x_2, \cdots, x_M\}$, each
associating with a validation loss function $f_i(\theta)$, where $\theta$ is the parameter of the
model. The overall validation loss would be,
\begin{align}
  G(\theta) = \frac{1}{M} \sum_{i=1}^M f_i(\theta).
\end{align}
Now, suppose we have another $N-M$ training data, $\{x_{M+1}, x_{M+2}, \cdots, x_N\}$, and we add
those validation data into this set to form our large training dataset $T$, which has $N$ data in
total. The overall training loss would be,
\begin{align}
  F(\theta) = \frac{1}{M} \sum_{i=1}^N f_i(\theta).
\end{align}

For simplicity, since $M \ll N$, we assume that the validation data is a subset of the training
data. During training, we take a mini-batch $B$ of training data at each step, and $|B| = n$.
Following some similar derivation as Appendix \ref{sec:mlp_derive}, we have the following update
rules:
\begin{align}
\theta_{t+1} = \theta_t - \frac{\alpha_t}{n} \sum_{i \in B} \max \left\{\nabla G^\top \nabla f_i, 0
\right\} \nabla f_i,
\label{eq:updaterule}
\end{align}
where $\alpha_t$ is the learning rate at time-step $t$. Since all gradients are taken at $\theta_t$,
we omit $\theta_t$ in our notations.

Since the validation loss $G(\theta)$ is Lipschitz-smooth, we have
\begin{align}
G(\theta_{t+1}) \leq G(\theta_t) + \nabla G^\top\Delta \theta + \frac{L}{2} \lVert \Delta \theta \rVert^2.
\end{align}

Plugging our updating rule (Eq. \ref{eq:updaterule}),
\begin{align}
G(\theta_{t+1}) \leq G(\theta_t) - I_1 + I_2,
\end{align}
where,
\begin{align}
\begin{split}
I_1 &=  \frac{\alpha_t}{n} \sum_{i \in B} \max \{\nabla G^\top \nabla f_i, 0 \}\nabla G^\top \nabla f_i\\
&=  \frac{\alpha_t}{n} \sum_{i \in B} \max \{\nabla G^\top \nabla f_i, 0 \}^2,\\
\end{split}
\end{align}
and,
\begin{align}
I_2 &= \frac{L}{2} \left \lVert \frac{\alpha_t}{n} \sum_{i \in B} \max \{\nabla G^\top \nabla f_i, 0 \}\nabla f_i \right \rVert^2\\
    &\le  \frac{L}{2} \frac{\alpha^2_t}{n^2} \sum_{i \in B} \left \lVert
          \max \left \{\nabla G^\top \nabla f_i, 0 \right \}\nabla f_i \right \rVert^2\label{eq:firstineq}\\
    &=    \frac{L}{2} \frac{\alpha^2_t}{n^2} \sum_{i \in B}
          \max \left \{\nabla G^\top \nabla f_i, 0 \right \}^2
          \left \lVert \nabla f_i \right \rVert^2\\
    &\le  \frac{L}{2} \frac{\alpha^2_t}{n^2} \sum_{i \in B}
          \max \left \{\nabla G^\top \nabla f_i, 0 \right \}^2 \sigma^2\label{eq:secondieq}.
\end{align}

The first inequality (Eq.~\ref{eq:firstineq}) comes from the triangle inequality. The second
inequality (Eq.~\ref{eq:secondieq}) holds since $f_i$ has $\sigma$-bounded gradients. If we denote
$\mathcal{T}_t = \sum_{i \in B} \max
\{\nabla G^\top \nabla f_i, 0 \}^2$, where $t$ stands for the time-step $t$, then
\begin{align}
\label{eq:decreasingvalid}
G(\theta_{t+1}) \leq G(\theta_t) - \frac{\alpha_t}{n} \mathcal{T}_t
                     \left(1 - \frac{L \alpha_t \sigma^2}{2 n } \right).
\end{align}

Note that by definition, $\mathcal{T}_t$ is non-negative, and since $\alpha_t \le
\frac{2n}{L\sigma^2}$, if follows that that $G(\theta_{t+1}) \le G(\theta_t)$ for any $t$.

Next, we prove $\EEsub{\mathcal{T}_t}{t}= 0$ if and only if $\nabla G = 0$, and
$\EEsub{\mathcal{T}_t}{t} > 0$ if and only if $\nabla G \neq 0$, where the expectation is taken over
all possible training batches at time step $t$. 
It is obvious that when $\nabla G = 0$, $\EEsub{\mathcal{T}_t}{t} = 0$. If $\nabla G \neq
0$, from the inequality below, we firstly know that there must exist a validation example $x_{j, 0 \leq j \leq M}$ such
that $\nabla G^\top \nabla f_j > 0$,
\begin{align}
\label{eq:positivedotprodexist}
0 < \lVert \nabla G \rVert^2 = \nabla G^\top \nabla G = \frac{1}{M} \sum_{i=1}^M \nabla G^\top \nabla f_i.
\end{align}
Secondly, there is a non-zero possibility $p$ to sample a training batch $B$ such that it contains this data $x_j$.
Also noticing that $\mathcal{T}_t$ is a non-negative random variable, we have,
\begin{align}
\label{eq:EETlargerzero}
\begin{split}
\EEsub{\mathcal{T}_t}{t} &\geq p \sum_{i \in B} \max\{\nabla G^\top \nabla f_i, 0\}^2\\
&\geq p \max\{\nabla G^\top \nabla f_j, 0\}^2 \\
&= p \left( \nabla G^\top \nabla f_j \right)^2 > 0.
\end{split}
\end{align}

Therefore, if we take expectation over the training batch on both sides of Eq. \ref{eq:decreasingvalid},
we can conclude that,
\begin{align}
\EEsub{G(\theta_{t+1})}{t} \leq G(\theta_{t}),
\end{align}
where the equality holds if and only if $\nabla G = 0$. This finishes our proof for Lemma 1.
\end{proof}

\section{Convergence rate of our method}
\label{sec:thmproof}
This section provides proof for Theorem 2.

\begin{thm*}
Suppose $G$, $f_i$ and $\alpha_t$ satisfy the aforementioned conditions, then
Algorithm~\ref{alg:ad} achieves $\EE{\lVert \nabla G(\theta_t) \rVert^2} \leq \epsilon$ in
$O(1/\epsilon^2)$ steps. More specifically,
\begin{align}
\min\limits_{0 < t < T} \EE{\lVert \nabla G(\theta_t) \rVert^2} \leq \frac{C}{\sqrt{T}},
\end{align}
where $C$ is some constant independent of the convergence process.

\end{thm*}

\begin{proof}
From the proof of Lemma~\ref{lem:convergence}, we have
\begin{align}
\begin{split}
&\frac{\alpha_t}{n} \left(1 - \frac{L \alpha_t \sigma^2}{2n} \right) \EEsub{\mathcal{T}_t}{0 \sim t} \\
\leq& \EEsub{G(\theta_t)}{0 \sim t-1} - \EEsub{G(\theta_{t+1})}{0 \sim t}.
\end{split}
\end{align}

If we let $\alpha_t$ to be a constant $\alpha < \frac{2n}{L\sigma^2}$ (or a decay positive sequence
upper bounded by $\alpha$), and let $\kappa = \left(1 - \frac{L \alpha \sigma^2}{2n} \right)\alpha/n
> 0$, then we have,
\begin{align}
\begin{split}
 \kappa \sum_{t=0}^T \EEsub{\mathcal{T}_t}{0 \sim t} &\leq \EEsub{G(\theta_0)}{0} - \EEsub{G(\theta_{T+1})}{0 \sim T}\\
  &\leq G(\theta_0) - G(\theta^*),
\end{split}
\end{align}

where $G(\theta^*)$ is the global minimum of function $G$. Therefore, it is obvious to see that
there exist a time-step $ 0 \leq \tau \leq T$ such that,
\begin{align}
\EEsub{\mathcal{T}_{\tau}}{0 \sim \tau} \leq \frac{G(\theta_0) - G(\theta^*)}{\kappa T}.
\end{align}

We next prove that for this time-step $\tau$, the gradient square $\EEsub{\lVert \nabla G(\theta_{\tau})
\rVert^2}{0 \sim \tau-1}$ is smaller than $O(1/\sqrt{T})$. 
Considering such $M$ training batches $B_1, B_2, \cdots, B_M$ such that $B_i$ is guaranteed to contain $x_i$. We know that those batches have non-zero sampling probability, denoted as $p_1, p_2, \cdots, p_M$. We also denote $p = \min\{p_1, p_2, \cdots, p_M\}$.
Now, we have,
\begin{align}
M\EEsub{\mathcal{T}_{\tau}}{0 \sim \tau}
&= \EEsub{M\EEsub{\mathcal{T}_{\tau}}{\tau}}{0 \sim \tau-1}\\
&\geq \EEsub{\sum_{k = 1}^M p_k \sum_{i \in B_k} \max\{\nabla G^\top \nabla f_i, 0\}^2}{0 \sim \tau-1} \label{ineq:first} \\
&\geq p\EEsub{\sum_{k = 1}^M \sum_{i \in B_k} \max\{\nabla G^\top \nabla f_i, 0\}^2}{0 \sim \tau-1} \\
&\geq p\EEsub{\sum_{i=1}^M \max\{\nabla G^\top \nabla f_i, 0\}^2}{0 \sim \tau-1} \\
&= p \sum_{i=1}^M \EEsub{\max\{\nabla G^\top \nabla f_i, 0\}^2}{0 \sim \tau-1}       \\
&\geq p \sum_{i=1}^M \left (\EEsub{\max\{\nabla G^\top \nabla f_i, 0\}}{0 \sim \tau-1} \right)^2 \label{ineq:second}     \\
&\geq \frac{p}{M} \left ( \sum_{i=1}^M \EEsub{\max\{\nabla G^\top \nabla f_i, 0\}}{0 \sim \tau-1} \right)^2\label{ineq:third}.
\end{align}
The inequality in Eq.~\ref{ineq:first} comes from the non-negativeness of $\mathcal{T}_t$, the
inequality in Eq.~\ref{ineq:second} comes from the property of expectation, and the final inequality
in Eq.~\ref{ineq:third} comes from the Cauchy-Schwartz inequality. Therefore,
\begin{align}
\begin{split}
&\sum_{i=1}^M \EEsub{\max\{\nabla G^\top \nabla f_i, 0\}}{0 \sim \tau-1} \\
\leq& M\sqrt{\frac{(G(\theta_0) - G(\theta^*))}{p \kappa}} \sqrt{\frac{1}{T}},
\end{split}
\end{align}
and so,
\begin{align}
\label{eq:finalequa}
\begin{split}
&\EEsub{\lVert \nabla G(\theta_{\tau}) \rVert^2}{0 \sim \tau-1}\\
=& \EEsub{\nabla G^\top \nabla G}{0 \sim \tau-1}\\
=& \EEsub{\nabla G^\top \left( \frac{\sum_{i=0}^M \nabla f_i}{M} \right)}{0 \sim \tau-1}\\
\leq&  \frac{1}{M} \sum_{i=1}^M \EEsub{\max\{\nabla G^\top \nabla f_i, 0\}}{0 \sim \tau-1}\\
\leq&  \sqrt{\frac{G(\theta_0) - G(\theta^*)}{p \kappa}} \sqrt{\frac{1}{T}}.
\end{split}
\end{align}
Therefore, we can conclude that conclude that our algorithm can always achieve $\min \limits_{0 < t < T} \EE{\lVert
\nabla G(\theta_t)\rVert^2} \leq O(\sqrt{1/T})$ in $T$ steps, and this finishes our proof of
Theorem~\ref{thm:convergencerate}.
\end{proof}

\fi

\end{document}